\documentclass{article}

\PassOptionsToPackage{numbers, compress, sort}{natbib}

\usepackage[preprint]{neurips_2026}

\usepackage[T1]{fontenc}    \usepackage{url}            \usepackage{booktabs}       \usepackage{amsfonts}       \usepackage{nicefrac}       \usepackage{microtype}      \usepackage[dvipsnames]{xcolor} 

\usepackage{amsmath}
\usepackage{amsthm}
\usepackage{thmtools}
\usepackage{hyperref}       \usepackage{amssymb}
\usepackage{amsfonts}
\usepackage{soul}
\usepackage{mathtools}
\usepackage{float}
\usepackage{cancel}
\usepackage{mathabx}
\usepackage{dsfont}
\usepackage{dirtytalk}
\usepackage{enumitem}
\usepackage[dvipsnames]{xcolor}
\usepackage[textsize=tiny]{todonotes}

\usepackage{tikz}
\usetikzlibrary{arrows.meta,patterns,calc}

\usepackage{algorithm}
\usepackage[noend]{algpseudocode}
\algrenewcommand\algorithmiccomment[1]{\hfill\(\triangleright\) \textit{#1}}

\usepackage{tcolorbox}

\newcommand{\N}{\mathbb{N}}

\newcommand{\Q}{\mathbb{Q}}
\newcommand{\R}{\mathbb{R}}

\newcommand{\E}{\mathbb{E}}

\newcommand{\Rd}{\mathbb{R}^{d}}

\newcommand{\SecMomTens}{\mathcal{C}_{t}}

\newcommand{\DGS}[1]{\todo[backgroundcolor=black,textcolor=white]{D\#: #1}}
\newcommand{\dgstext}[1]{\textcolor{blue}{#1}}
\newcommand{\DGSi}[1]{\todo[inline,backgroundcolor=black,textcolor=white]{D\#: #1}}

\newcommand{\ymm}[1]{{#1}} 

\newcommand{\pt}[1]{{#1}}

\newcommand{\dd}{\mathrm{d}}

\newcommand{\bigO}[1]{\mathcal{O}({#1})}

\DeclareMathOperator{\supp}{\textup{supp}}

\DeclareMathOperator{\closedconv}{\overline{\textup{conv}}}

\newcommand{\step}[1]{\textbf{Step}~\ensuremath{\mathbf{#1}}}

\newcommand{\CovarT}{\Sigma_{t}}
\newcommand{\dCovarT}{\dot{\Sigma}_{t}}
\newcommand{\ddCovarT}{\ddot{\Sigma}_{t}}

\definecolor{myblue}{rgb}{0.8,0.88,1}
\presetkeys{todonotes}{color=myblue}{}

\tikzset{
  axis/.style={very thick},
  pulse/.style={draw=black, fill=gray!20, line width=0.9pt},
  streamline/.style={draw=black!55, line width=0.9pt},
  gap/.style={draw=red!70!black, line width=1.0pt,
              pattern=north east lines, pattern color=red!65},
  bluepath/.style={draw=cyan!60!blue, line width=2.3pt},
  note/.style={font=\small},
  title/.style={font=\bfseries}
}

\newcommand{\pulseRightAt}[4]{\path[pulse]
    (#1-#4, #2-#3) rectangle (#1, #2+#3);
}
\newcommand{\pulseLeftAt}[4]{\path[pulse]
    (#1, #2-#3) rectangle (#1+#4, #2+#3);
}
\newcommand{\filledBellAtLeft}[4]{\path[pulse]
    (#1,#2-3*#3)
    -- plot[smooth,variable=\y,domain={#2-3*#3}:{#2+3*#3}]
       ({#1 - #4*exp(-((\y-#2)^2)/((#3)^2))}, \y)
    -- (#1,#2+3*#3)
    -- cycle;
}
\newcommand{\filledBellAtRight}[4]{\path[pulse]
    (#1,#2-3*#3)
    -- plot[smooth,variable=\y,domain={#2-3*#3}:{#2+3*#3}]
       ({#1 + #4*exp(-((\y-#2)^2)/((#3)^2))}, \y)
    -- (#1,#2+3*#3)
    -- cycle;
}

\newtheorem{theorem}{Theorem}[section]
\newtheorem*{theorem*}{Theorem}
\newtheorem{corollary}[theorem]{Corollary}

\newtheorem{lemma}[theorem]{Lemma}

\theoremstyle{definition}
\newtheorem{definition}[theorem]{Definition}

\newtheorem*{assumption*}{Assumption}
\newtheorem{open-problem}[theorem]{Open Problem}

\theoremstyle{remark}
\newtheorem{remark}[theorem]{Remark}
\newtheorem{example}[theorem]{Example}

\title{One-Shot Generative Flows:\\ Existence and Obstructions}

\author{
  Panos Tsimpos \\ 
  Operations Research Center \\
  Massachusetts Institute of Technology \\
  Cambridge, MA 02139 \\
  \texttt{ptsimpos@mit.edu}
  \And
  Daniel Sharp\\
  Center for Computational Science \& Engineering \\
  Massachusetts Institute of Technology \\
  Cambridge, MA 02139 \\
  \texttt{dannys4@mit.edu}
  \And
  Youssef Marzouk\\
  Laboratory of Information and Decision Systems \\
  Massachusetts Institute of Technology \\
  Cambridge, MA 02139 \\
  \texttt{ymarz@mit.edu}
}

\begin{document}

\maketitle

\begin{abstract}
  We study dynamic measure transport for generative modeling, focusing on transport maps that connect a source measure $P_0$ to a target measure $P_1$ by integrating a velocity field of the form $v_t(x) = \E[\dot X_t \mid X_t = x]$, where $X_\bullet = (X_t)_t$ is a stochastic process satisfying $(X_0,X_1)\sim{P_0}\otimes{P_1}$ and $\dot X_t$ is its time derivative.
  We investigate when $X_\bullet$ induces a \emph{straight-line flow}: a flow whose pointwise acceleration vanishes and is therefore exactly integrable by any first-order method.
  First, we develop multiple characterizations of straight-line flows in terms of PDEs involving the conditional statistics of the process.
  Then, we prove that straight-line flows under endpoint independence exhibit a sharp dichotomy.
  On the one hand, we construct explicit, computable straight-line processes for arbitrary Gaussian endpoints. On the other hand, we show that straight-line processes do not exist for targets with sufficiently well-separated modes. We demonstrate this obstruction through a sequence of increasingly general impossibility theorems that uncover a fundamental relationship between the sample-path behavior of a process with independent endpoints and the space-time geometry of this process' flow map. Taken together, these results provide a structural theory of when straight-line generative flows can, and cannot, exist.
\end{abstract}

\section{Introduction}
Sampling from challenging probability distributions is central to probabilistic inference and modern generative modelling. A recent line of work establishes \emph{dynamic measure transport} as a unifying paradigm: First, construct a stochastic process $(X_t)_{t\in[0,1]}$ with marginals that interpolate a source distribution $P_0$ at $t=0$ and a target distribution $P_1$ at $t=1$. Then, estimate the \emph{conditional velocity}
\[
v_t(x) \coloneq \mathbb{E}[\dot X_t \mid X_t=x]
\]
and generate samples by evaluating the ordinary differential equation (ODE) flow maps $\phi_t$ defined by
\begin{equation}\label{eqn:def_flow_map_ode}
  \partial_{t} \phi_t(x) = v_t\!\left(\phi_t(x)\right) \;\; \text{with} \;\; \phi_0(x)=x \, .
\end{equation}
This perspective underlies methods such as \emph{stochastic interpolants}~\cite{albergo2023stochastic,albergo2022building}, \emph{flow matching}~\cite{lipman2022flow,lipman2024flow}, and \emph{score-based probability flow ODEs}~\cite{song2021scorebased, chen2023probability}, as well as \emph{rectified flows}~\cite{liu2022flow,liu2022marginalpreserving}---all of which have demonstrated strong empirical performance.
The computational efficiency of these methods hinges on the geometry of the induced flow. Generic flows demand many evaluations of the conditional velocity $v_t$, as numerical integration error scales with the curvature (and higher derivatives) of $\phi_t$. In contrast, if the flow is straight, meaning that the acceleration of the flow map vanishes,
\[
\partial_{tt} \, \phi_t(x) \equiv 0 \quad \text{for all } (x,t),
\]
then $\phi_t(x)$ is affine in $t$, expressed as $\phi_t(x)=(1-t)x+t\,\phi_1(x)$. Thus, any first-order integrator is exact; one can traverse the entire path with a single velocity evaluation. This motivates a fundamental question: \emph{Which stochastic processes $(X_t)_t$ with $X_0\!\sim\! P_0$ and $X_1\!\sim\! P_1$ induce straight flows?}
Prior work offers algorithmic frameworks that convert an arbitrary flow into a straight one (e.g., \cite{liu2022flow,liu2022marginalpreserving}) \pt{or analyzes specific algorithms~\cite{bansal2025wasserstein,hertrich2025relation}.} Yet a structural characterization of when straightness is {intrinsic} to the underlying process has remained open. This paper develops such a theory.

One way to trivially obtain a straight flow is via an optimal transport coupling; see~\cite{mccann1997convexity, chewi2025statistical}.
Namely, let $T: \Rd \to \Rd$ be such that $T_\# P_0 = P_1$ \ymm{solves the Monge problem with appropriate cost} and set
\begin{equation}
  \label{eq:OT-driven-FM}
  X_t = (1-t) \, X_0 + t \, T(X_0) \;\; \textup{with} \;\; X_0 \sim P_0 \, .
\end{equation}
It is not hard to see that the resulting flow is straight, owing to the fact that each sample path $t \mapsto X_t(\omega)$ is straight and no two sample paths intersect. In fact, any transport map $T$ with positive definite Jacobian $\nabla T(x)$ at each $x \in \Rd$ yields such a trivial solution; see Theorem~\ref{thm:affine_processes} for details. Examples include Monge maps corresponding to strictly convex potentials and Knothe--Rosenblatt rearrangements between densities with mild regularity~\cite{marzouk2024distribution}.
Yet such a solution is circular, as it constructs a process $X_\bullet$ to sample $P_1$ in terms of a transport map $T$. If we had access to $T$, we could sample $P_1$ by sampling $P_0$ and applying $T$ without ever constructing $X_\bullet$.

Recent literature has successfully built on this idea by developing sampling algorithms that construct near-optimal couplings to produce nearly straight flows. Indeed, rectified-flow methods, minibatch-coupling methods, and optimal-transport variants of conditional flow matching exploit this fundamental property of optimal transport, yielding practically relevant algorithms~\cite{liu2022flow,pooladian2023multisample,tong2024minibatchot}.
Such algorithms, however, require additional up-front computation: first estimate some surrogate $\tilde{T}$ to a transport map $T$ and then instantiate a simple process $X_\bullet$ in terms of $\tilde{T}$. \ymm{A different set of approaches, e.g., self-distillation \cite{salimans2022progressive, boffi2025build} and mean flow \cite{geng2025mean,geng2025improved}, building on the notion of consistency models \cite{song2023consistency,kim2023consistency,boffi2024flow}, aim to learn the ODE flow map $\phi$ directly, but doing so requires more complex training objectives.}

Here, instead, we explore straight-line flows that do not rely on
\ymm{the explicit construction of a transport map or any other non-trivial coupling.}
Indeed, it is natural to collect two sets of i.i.d.\ samples from $P_0$ and $P_1$ and to construct the process $X_\bullet$ by connecting independent pairs $(X_0, X_1)$. In this setting, there is no relationship between source and target samples; the only coupling we can impose {a priori}, i.e., without any computation, is the product measure $P_0 \otimes P_1$. We thus require that the candidate process $X_\bullet$ have independent endpoints and pose the central question of this paper:
\begin{tcolorbox}[colback=white,colframe=black,left=1mm,right=1mm,
  top=1mm,bottom=1mm]
\emph{Given measures $P_0$ and $P_1$, which processes $X_\bullet$ with $(X_0, X_1) \sim P_0 \otimes P_1$ induce straight flows?}
\end{tcolorbox}

To study this problem, we introduce two classes of stochastic processes, defined formally in \S\ref{sec:prelims}.  The first, $\mathcal{S}_{\mathrm{SL}}(P_0,P_1)$, consists of suitably regular processes that have independent endpoints and induce straight-line flows. The second, $\mathcal{F}_{\mathrm{SL}}(P_0,P_1)$, consists of those members of $\mathcal{S}_{\mathrm{SL}}(P_0,P_1)$ that admit a computationally tractable representation called a \emph{generalized interpolant}, an extension of the stochastic interpolants framework~\cite{albergo2023stochastic}. This distinction lets us separate theoretical existence from computational tractability.

Our \textbf{main contributions} are as follows:
\begin{enumerate}
    \item \textbf{Analytical characterization of straight-line processes.} We derive several PDE characterizations of straight-line flows induced by stochastic processes, \ymm{linking the Eulerian and Lagrangian perspectives}.
As a consequence, we show that the common processes \( X_t = (1-t)X_0 + tX_1\)
    induce straight flows if and only if the endpoint coupling is deterministic.

    \item \textbf{Existence of straight-line processes under endpoint independence.}
For arbitrary Gaussian source and target distributions, we construct explicit, computable processes with independent endpoints, as visualized in Figure~\ref{fig:straight-lines-gauss}.

  \item \textbf{Obstructions for multi-modal targets.} We prove that straightness under endpoint independence fails in a broad class of multi-modal settings; \ymm{in other words, strong modal separation is fundamentally incompatible with the generation of straight-line flows from stochastic processes. Our analysis also brings new analytical techniques---combining geometry, topology, and stochastic process theory---to the study of flow-based generative models.}

  \end{enumerate}

\begin{figure}[h]
    \centering
    \includegraphics[clip, trim={2cm 1cm 2cm 0}, width=\textwidth]{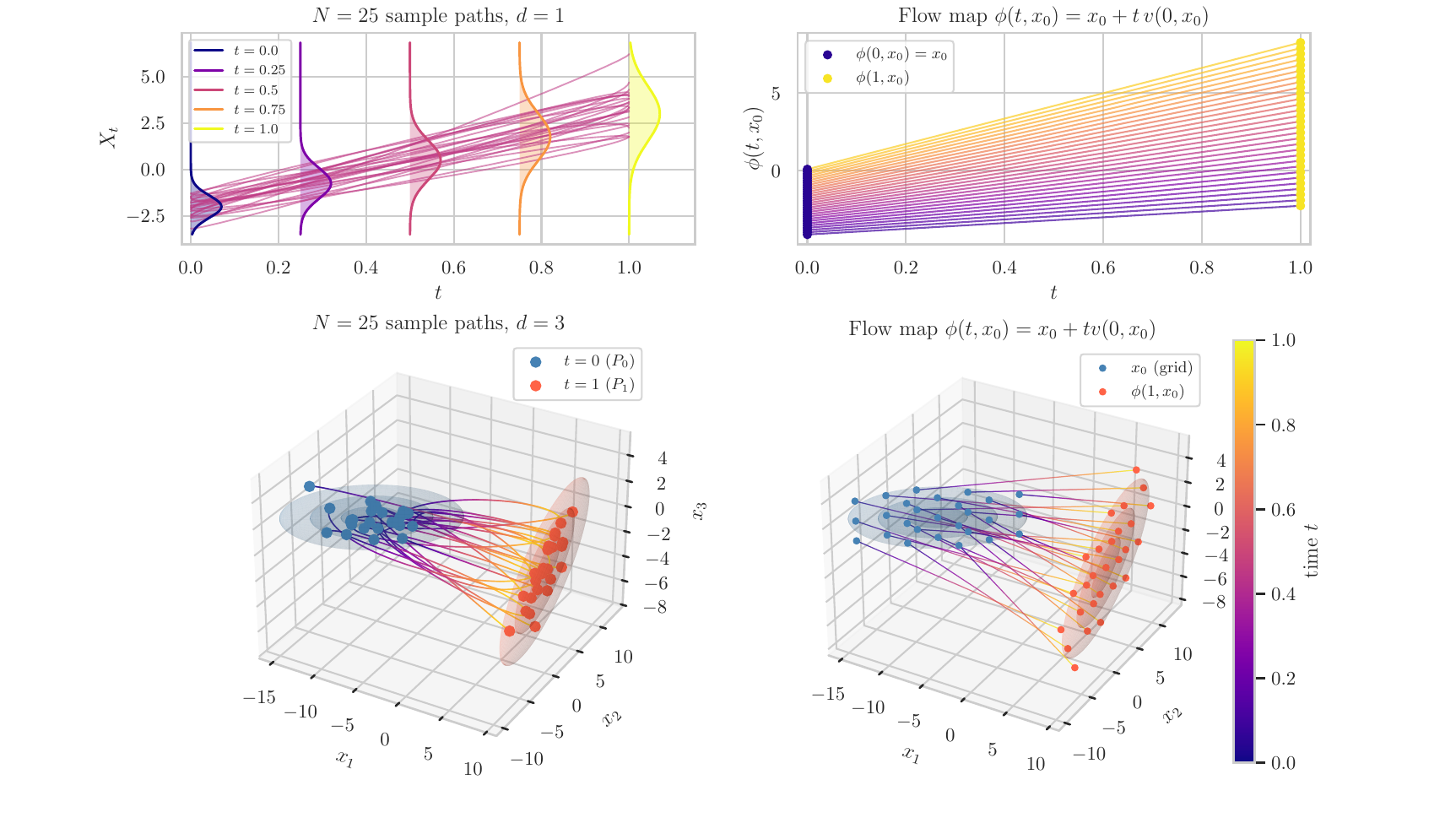}
    \caption{
      Straight-line Gaussian process for $d = 1$ (top) and $d=3$ (bottom). We construct these via Theorem~\ref{thm:straight-line-gaussians-multivariate}.
      \emph{Left:} $N = 25$ sample paths $X_t$ in space-time; for $d=1$, we show marginal
      Gaussian densities overlaid at $t \in \{0, 0.25, 0.5, 0.75, 1\}$.
      \emph{Right:} flow map $\phi(t, x_0) = x_0 + t\,v(0,x_0)$ initialized on
      a regular grid.}
    \label{fig:straight-lines-gauss}
\end{figure}

\subsection{Preliminaries}\label{sec:prelims}

We now define key objects used below; Appendix~\ref{appendix:notation} defines additional notation.

\begin{definition}\label{def:S-P_0-P_1}
  An $\Rd$-valued stochastic process $X_\bullet \coloneq (X_t)_{t \in [0,1]}$ is said to have \emph{absolutely continuous sample paths} if the function $t \mapsto X_t$ belongs to $W^{1,1}([0,1]; \Rd)$ $\mathbb{P}$-a.s.
  Given measures $P_0$ and $P_1$ on $\Rd$, we write the set $\mathcal{S}(P_0, P_1)$ as all stochastic processes with absolutely continuous sample paths satisfying: (i)
  \((X_0, X_1) \sim P_0 \otimes P_1\) ; (ii) $\E \left[ \| X_0\| + \int_0^1 \| \dot X_t \| dt \right] < \infty$; (iii)
  the conditional velocity
  $v(t, x) = \E \left[ \dot X_t | X_t = x \right]$ is bounded and spatially Lipschitz uniformly in time, i.e., $\exists L > 0$ s.t.
  \begin{equation}
    \label{eq:lipschitz-v}
    \| v(t, x) - v(t, y) \| \leq L \, \| x - y \|  \;\; \forall \, (t, x, y) \in [0,1] \times \Rd \times \Rd \, .
  \end{equation}
\end{definition}
\begin{definition}
  Fix $p \in [1, 2)$. A \emph{base interpolation path} is a Borel-measurable mapping
  $ F: [0,1] \times \R^{2d + \ell} \to \Rd$
such that, for $(x,y,z)$-Lebesgue a.e. points in $\R^{2d + \ell}$, the mapping $t \mapsto F(t, x, y, z)$ is in $W^{1, p}\left([0,1] ; \Rd \right)$ with boundary conditions $F(0, x, y, z) = x$ and $F(1, x, y, z) = y$. We write $\mathcal{F}$ as the set of all such mappings.
\end{definition}

\begin{definition}\label{def:generalized_interpolant}
  A \emph{generalized interpolant} is a pair $(F, Q)$ with $F \in \mathcal{F}$ and measure $Q \in \mathcal{P}(\R^\ell)$ for some $\ell \in \N$. For measures $P_0, P_1 \in \mathcal{P}(\Rd)$, the induced stochastic process of $(F, Q)$ is $X_\bullet: t \mapsto X_t$ defined for $(X_0, X_1, Z) \sim P_0 \otimes P_1 \otimes Q$ as
  \begin{equation*}
    X_t = F(t, X_0, X_1, Z) \, .
  \end{equation*}
\end{definition}

\begin{definition}
  \label{def:straight-line-process-and-straight-line-interpolant}
  Fix measures $P_0, P_1 \in \mathcal{P}(\Rd)$.
  A \emph{straight-line process} connecting $P_0$ and $P_1$ is a process $X_\bullet \in \mathcal{S}(P_0, P_1)$ such that the 
the flow $\phi$ of the conditional velocity $v$ defined in \eqref{eqn:def_flow_map_ode} 
  satisfies
  \begin{equation}
    \label{eq:no-acceleration-condition}
    \partial^2_t \phi(t, x) = 0 \quad \text{for all } (t, x) \in [0,1] \times \Rd \, .
  \end{equation}
Finally, define
  \begin{align}
    \label{eq:def-S_SL}
    & \mathcal{S}_{\textup{SL}}(P_0, P_1) \coloneq \Big\{ X_\bullet \in \mathcal{S}(P_0, P_1) : X_\bullet \text{ satisfies }  \eqref{eq:lipschitz-v} \textup{ and } \eqref{eq:no-acceleration-condition} \Big\} \\
    \label{eq:def-F_SL}
    & \mathcal{F}_{\textup{SL}}(P_0, P_1) \coloneq \Big\{ (F, Q) \in \mathcal{F} \times \mathcal{P}(\R^\ell) : X_\bullet \in \mathcal{S}_{\textup{SL}}(P_0, P_1) \Big\} \, .
  \end{align}
\end{definition}

\section{PDE characterization of straight-line flows}
\label{sec:PDE_characterization}

In this section, we make additional regularity assumptions on the flow maps $(\phi_t)_{t \in [0,1]}$ and the processes $X_\bullet$ in order to develop an analytical theory of straight-line flows. Recalling the notation from Definition~\ref{def:straight-line-process-and-straight-line-interpolant}, we characterize processes in $\mathcal{S}_{\textup{SL}}(P_0, P_1)$ via partial differential equations (PDEs) involving  the conditional velocity $v_t$ and other statistics of $X_\bullet$. All proofs are in Appendix~\ref{subsec:proofs-PDE_characterization}.

\begin{restatable}{assumption}{RegAssumption}
  \label{assumptoion:regularity}
  For the remainder of this section fix $P_0, P_1$ and take a process $X_\bullet \in \mathcal{S}(P_0, P_1)$. 
  Assume that the marginal densities $\rho_t$ of $X_t$ exist for all $t \in [0,1]$, assume $\E \| X_t \|^2 < \infty$
  and that the sample paths $t \mapsto X_t$ are in $W^{2,1}([0,1];\Rd)$ a.s.
  Moreover, assume that the trajectory maps $t \mapsto \phi_t(x)$ are in $C^2([0,1];\Rd)$ for each $x \in \Rd$, whereas the flow maps $x \mapsto \phi_t(x)$ are homeomorphisms for each $t \in [0,1]$.
  Lastly, assume the conditional velocity $v: (t,x) \mapsto v_t(x)$ is continuously differentiable jointly over $[0,1] \times \Rd$.
\end{restatable}
Under these regularity assumptions, we can define the \emph{material derivative} of the velocity $v$ at $x \in \Rd$ as $D_t \, v(t,x) \coloneq \partial_t \, v(t,x) + \left(v(t,x) \cdot \nabla\right) \, v(t,x)$.

\begin{restatable}{proposition}{PropReformulation}
\label{prop:reformulaton-1}
   Consider the conditional velocity $v_t(x) = \E[\dot X_t \mid X_t = x]$. The following are equivalent for all $(t, x) \in [0,1] \times \Rd$:
   \[\textrm{1. }\partial_{tt} \, \phi_t(x) = 0,\quad\textrm{2. }\phi_t(x) = (1-t) \, x + t \, \phi_1(x),\quad\textrm{3. }D_t v_t(x) = 0.
   \]
\end{restatable}
This result shows equivalent descriptions of straightness: the \emph{Lagrangian}, focusing on the flow map $\phi$, and the \emph{Eulerian}, focusing on the velocity field $v$.
Next, we relate the material derivative to other statistical (and Eulerian) quantities of interest. Define the \emph{acceleration}, \emph{second moment} velocity tensor, and \emph{covariance} or \emph{Reynolds stress} tensor as, respectively,
\[
  a(t,x) \coloneq \mathbb{E} \bigl[ \ddot X_t \, \vert \,  X_t = x \bigr] , \
  \mathcal{C}(t,x) \coloneq \mathbb{E} \bigl[  \dot X_t \otimes \dot X_t \, \vert \,  X_t = x \bigr], \
  \Pi(t,x) \coloneq \mathcal{C}(t,x) - v(t,x) \otimes v(t,x) \, .
\]
We can then elucidate the connection between the covariance $\Pi_t$ and the material derivative $D_t v_t$.
\begin{restatable}{lemma}{LemmaReynolds}
    \label{lemma:reynolds-relation}
    The following equation holds for all $(t,x) \in [0,1] \times \Rd$:
    \begin{equation*}
      \rho_t \, D_t v_t + \nabla \cdot \left( \rho_t \, \Pi_t \right) = \rho_t \, a_t \, .
    \end{equation*}
\end{restatable}
\begin{corollary}
  \label{corollary:straightness}
  For all $(t,x) \in [0,1] \times \Rd$, a process $X_\bullet$ satisfying Assumption~\ref{assumptoion:regularity} is a straight-line process if and only if it satisfies the equivalent PDEs
  \begin{equation}\label{eqn:reformulation_straightness_PDE_equiv}
    X_\bullet \in\mathcal{S}_{\mathrm{SL}}(P_0, P_1)\,\iff\,D_t v_t = 0 \, \iff \, \nabla \cdot \left( \rho_t \, \Pi_t \right) = \rho_t \, a_t \, .
  \end{equation}
\end{corollary}
\begin{remark}\label{remark:reduced_regularit}
  The first equivalence \ymm{also} holds for a process with sample paths a.s.~in $W^{1,1}([0,1];\Rd)$.
\end{remark}

\subsection{Linear characteristics}
\label{subsec:linear_characteristics}
Now we use Corollary~\ref{corollary:straightness} to obtain a complete characterization of affine processes inducing straight flows.
We define an \emph{affine process} with marginals $P_0, P_1 \in \mathcal{P}(\Rd)$ to have the form
\begin{equation*}
    X_t = (1-t) \, X_0 + t \, X_1 \, ,
\end{equation*}
where $(X_0, X_1) \sim \gamma$ and $\gamma \in \mathcal{P}(\Rd \times \Rd)$ is a coupling of the measures $P_0$ and $P_1$, meaning that the first and second marginals of $\gamma$ are $P_0$ and $P_1$, respectively.
This process is widely used in computational statistics and generative modelling~\cite{lipman2022flow,liu2022marginalpreserving,albergo2023stochastic}. The main result of this section is that the flow generated by an affine process can only be straight if the coupling $\gamma$ between the endpoints is deterministic, i.e., if there is some measurable map $T: \Rd \to \Rd$ such that $T(X_0) = X_1$ almost surely. The key input is that the conditional acceleration $a_t (x)$ \emph{vanishes} for affine processes. Therefore, Corollary~\ref{corollary:straightness}
yields the straight-line characterization
$ \nabla \cdot \left( \rho_t \, \Pi_t \right) = 0 $.

\begin{restatable}{theorem}{ThmAffineProcesses}
  \label{thm:affine_processes}
  Suppose that $X_t = (1-t) X_0 + t \, X_1$ with $(X_0, X_1) \sim \gamma$ and $\gamma \in \mathcal{P}(\Rd \times \Rd)$ with first and second marginals $P_0$ and $P_1$, respectively.
Then
\begin{equation}
  \label{eq:suff-condition-for-transport-map}
  \nabla \cdot \left( \rho_t \, \Pi_t \right) = 0 \quad \forall \, t \in [0,1] \, ,
\end{equation}
  implies that there exists a measurable map $T: \Rd \to \Rd$
  such that $X_1 = T(X_0)$ a.s.
Conversely, if there exists an injective and continuously differentiable map $T: \Rd \to \Rd$ satisfying $X_1 = T(X_0)$ a.s.\ with everywhere positive semi-definite Jacobian, then~\eqref{eq:suff-condition-for-transport-map} holds.  
\end{restatable}
Any injective and continuously differentiable \emph{optimal transport} (OT) map $T$ has an everywhere positive semi-definite Jacobian (see Remark~\ref{rk:OT-map-gives-straight-lines}) so that $\Pi_t$ corresponding to $X_t = (1-t) X_0 + t \, T(X_0)$ satisfies \eqref{eq:suff-condition-for-transport-map}.
Thus, a regular enough OT map induces a straight-line flow.

\section{Existence}
\label{sec:existence}

In this section, we construct generalized interpolants that yield straight-line flows for arbitrary Gaussian endpoint measures.
Moreover, we demonstrate the importance of the Lipschitz condition in Definition~\ref{def:straight-line-process-and-straight-line-interpolant}. Without this condition, we can construct a straight-line process between any two measures that is singular at some time $\tau \in (0,1)$.
All proofs are in Appendix~\ref{subsec:proofs-existence}.

\subsection{Gaussian endpoint measures}\label{sec:gaussian-endpoints}
 We seek an explicit $X_\bullet \in \mathcal{F}_{\textup{SL}}(P_0, P_1)$ for arbitrary Gaussian marginals $P_0$ and $P_1$.
In the $d=1$ case, observe that a Gaussian process $X_t \sim \mathcal{N}(m_t, \sigma_t^2)$ satisfies the equivalence
    \begin{equation*}
        D_t v_t \equiv 0 \iff \begin{cases}
            \ddot m_t \equiv 0 \, ,\\
            \ddot\sigma_t \equiv 0 \, .
        \end{cases}
    \end{equation*}
    That is, a scalar-valued Gaussian process indexed by $t \in [0,1]$ yields a straight flow if and only if its mean and standard deviation are affine functions of $t$.
Therefore, we seek a generalized interpolant, cf.~Definition~\ref{def:generalized_interpolant}, that satisfies this condition. A natural choice is to pick $(X_0, X_1, Z)$ to be independent Gaussian variables, then to construct an affine combination of $X_0, X_1, Z$ with time-dependent coefficients; these coefficients are chosen to precisely enforce the above equivalence.

Proposition~\ref{prop:straight-line-gaussians-1d} illustrates this construction for $d=1$. Though the one-dimensional case is instructive, below we present the main result of this section, which generalizes this simpler construction to arbitrary non-degenerate Gaussian endpoints in any dimension $d \geq 1$. 
\begin{restatable}{theorem}{ThmMultivariateGaussians}
    \label{thm:straight-line-gaussians-multivariate}
    Let $P_0 = \mathcal{N}(m_0, \Sigma_0)$ and $P_1 = \mathcal{N}(m_1, \Sigma_1)$ be two non-degenerate Gaussian measures with $m_i \in \Rd$ and $\Sigma_i \in \R_{++}^{d \times d}$ positive definite for $i = 0,1$. Then, the generalized eigenvalues and eigenvectors of the matrix pencil $(\Sigma_1,\Sigma_0)$ induce, respectively, a diagonal matrix $\Lambda\in\R^{d\times d}_{++}$ and matrix $V\in\R^{d\times d}$ such that
    \begin{equation*}
        \Sigma_0 = VV^\top,\quad \Sigma_1 = V\Lambda V^\top.
    \end{equation*}
    For $\Sigma_Z = V \Lambda^{1/2} V^\top$, we have $\Sigma_Z \in\R_{++}^{d\times d}.$ Further, the generalized interpolant defined via
    \begin{equation*}
        F(t, x, y, z) = (1-t) \, x + t \, y + \sqrt{2 \, t \, (1 - t)} \, z \, ,\qquad
        Q = \mathcal{N}(0, \Sigma_Z)\,
    \end{equation*}
    satisfies
    \begin{equation*}
        (F, Q) \in \mathcal{F}_{\textup{SL}}(P_0, P_1) \, .
    \end{equation*}
\end{restatable}
In the case of commuting covariance matrices $\Sigma_0,\Sigma_1$, we can define the matrices $S_i = \Sigma_i^{1/2}$ as the respective symmetric square roots for $i=0,1$. Then, one can show $\Sigma_Z = S_0S_1$, a clear extension of the geometric mean used in the one-dimensional case of Proposition~\ref{prop:straight-line-gaussians-1d}. See also Theorem~\ref{thm:existence-of-straight-line-interpolants-between-gaussians}. We further remark that this is a generalization of the ``variance-preserving interpolant'' from, e.g.,~\cite{albergo2023stochastic}.

The interpolant above bears some similarity to a Brownian bridge process~\cite[Section 5.6.B]{karatzas2014brownian}. Specifically, for a Brownian motion $B_\bullet$, the standard Brownian bridge is defined by $\widehat{B}_t = B_t - t \, B_1$.
  Further, a Brownian bridge \emph{mixture} connecting $P_0, P_1\in\mathcal{P}(\R^d)$ can be defined by drawing $(X_0, X_1) \sim P_0 \otimes P_1$ and setting
\begin{equation} \label{eq:brownianbridge}
    \widehat{B}^{P_0, P_1}_t = (1-t) \, X_0 + t \, X_1 + \widehat{B}_t \, .
  \end{equation}
  Notice the similarity of \eqref{eq:brownianbridge} to the construction of $F$ in Theorem~\ref{thm:straight-line-gaussians-multivariate}. The key difference is that the correction term $\widehat{B}_t$ has covariance $t(1-t)I_d$ instead of $2t(1-t)\Sigma_Z$. As shown above, the latter choice ensures straightness.

\begin{remark}
  \label{rk:straight-line-flow-with-same-endpoints}
   Deducing that
$
    \mathcal{F}_{\textup{SL}}(P, P) \neq \emptyset
    $
for some $P \in \mathcal{P}(\Rd)$ is not a trivial problem.
  Indeed, one might expect that the constant process
  \begin{equation*}
    X_t \equiv X_0 \, , \quad \forall t \in [0,1] \, ,
  \end{equation*}
  lies in $\mathcal{F}_{\textup{SL}}(P, P)$. Unfortunately, this gives $(X_0, X_1) = (X_0, X_0)$: that is, the endpoints are not independent. Thus, $X_\bullet \notin \mathcal{F}_{\textup{SL}}(P,P)$. In fact, the only process $X_\bullet \in \mathcal{F}_{\textup{SL}}(P, P)$ that we are aware of  for $P = \mathcal{N}(m, \Sigma)$, with $m \in \Rd, \Sigma \in \R_{++}^{d \times d}$, uses the construction of Theorem~\ref{thm:straight-line-gaussians-multivariate}; see Corollary~\ref{eq:F_SL_nonempty_for_gaussians}.
\end{remark}

\subsection{Trivial solutions with non-Lipschitz velocity}
Here, we discuss the existence of a trivial almost-everywhere solution to the PDE $D_tv_t = 0$ \eqref{eqn:reformulation_straightness_PDE_equiv}.\begin{restatable}{proposition}{PropNonLipschitz}
    \label{prop:non-lipschitz-velocity}
    Choose any $\tau\in (0,1)$ and consider the process $X_t$ such that
    \begin{equation*}
        X_t = \alpha_t X_0 + \beta_t X_1,\quad \alpha_t = \begin{cases}1-\frac{t}{\tau} & t < \tau\\0 & t\geq \tau\end{cases},\quad \beta_t = \begin{cases}0 & t\leq \tau\\ \frac{t-\tau}{1-\tau} & t > \tau\end{cases} \, .
    \end{equation*}
    Then, sample trajectories $X_t$ satisfy $D_t v_t \equiv 0$ for all $t\in[0,1]\setminus\{\tau\}$, but the associated velocity field $v$ does not satisfy the Lipschitz assumption in Definition~\ref{def:straight-line-process-and-straight-line-interpolant}. Therefore, our paths are not straight, i.e., $X_\bullet \notin \mathcal{S}_{\textup{SL}}(P_0, P_1)$ for any $P_0, P_1$.
\end{restatable}
Intuitively, the above process $X_\bullet$ collapses to the origin at time $t = \tau$ and then re-expands. For $t\in [0,1]\setminus\{\tau\}$, no two characteristics cross, hence the covariance tensor $\Pi_t$ vanishes. 

We include Proposition~\ref{prop:non-lipschitz-velocity} as a cautionary tale: First, the equivalence between a straight-line and a one-shot flow is broken. Second, it addresses the natural question of ``two-shot'' flows, i.e., flows with piecewise linear characteristics: loosening the one-shot requirement can easily create an ill-posed problem with a velocity that is computationally impossible to integrate. This ill-posedness is ruled out by our Lipschitz assumption in Definition~\ref{def:straight-line-process-and-straight-line-interpolant}.

\section{Impossibility}
\label{sec:impossibility}

Next, we move beyond generalized interpolants for Gaussian measures and investigate more general
straight-line processes for arbitrary endpoint measures $P_0, P_1 \in \mathcal{P}(\Rd)$.
We show that, when the measures $P_0, P_1$ are ``multi-modal'' in a sense that will be made precise,
\begin{equation*}
  \mathcal{F}_{\textup{SL}} \subseteq \mathcal{S}_{\textup{SL}} = \emptyset \, .
\end{equation*}
(Recall \eqref{eq:def-S_SL} and~\eqref{eq:def-F_SL} for definitions.) 
To arrive at our main result, we investigate this problem at varying levels of generality. Specifically, we start with the case of $P_0 = P_1 = P$ and let $P$ have disconnected support in $d=1$; see Theorem~\ref{thm:non-existence-for-mixture-of-uniforms}. Then, we generalize to arbitrary dimension $d \geq 1$; see Theorem~\ref{thm:non-existence-mixture-of-uniforms-in-high-d}.
We finish with the case of $P_0 \neq P_1$ supported on the entire real line $\R$, with $P_0, P_1$ concentrated appropriately; see
Theorem~\ref{thm:non-existence-for-Gaussian-mixture}.
All proofs can be found in Appendix~\ref{subsec:proofs-impossibility}.

\subsection{Disconnected support}\label{subsec:disconnected-support}
We first show that, for any $P \in \mathcal{P}(\Rd)$ with disconnected support, the set $\mathcal{S}_{\textup{SL}}(P, P)$ defined in~\eqref{eq:def-S_SL} is empty. This serves as a good way to introduce the proof techniques used throughout this section. Moreover, it builds on Remark~\ref{rk:straight-line-flow-with-same-endpoints} in the previous section, where we noted that constructing straight-line processes between identical endpoint measures is nontrivial. In fact, our results show that it can be impossible.

\begin{definition}
  \label{def:disconnected-support}
  A measure $P \in \mathcal{P}(\Rd)$ has \emph{disconnected support} if there exist bounded open sets $S_0, S_1 \subseteq \Rd$ such that\footnote{
    For a set $S \subseteq \Rd$ we write $\closedconv S$ for the closed, convex hull of $S$. See Section~\ref{subsec:notation-analysis} for a precise definition.
  } $\closedconv \, S_0 \cap \closedconv \, S_1 = \emptyset$ and $\supp(P) = \overline{S_0} \cup \overline{S_1}$ with $P(S_i) > 0$ for $i = 0,1$. Furthermore, we say $P$ has a \emph{properly} disconnected support if, for each $i \in \{0,1\}$ and any Borel set $A \subseteq S_i$ with positive Lebesgue measure, we additionally have $P(A) > 0$.
\end{definition}

\begin{restatable}{theorem}{ThmNonExistMixUniforms}
  \label{thm:non-existence-for-mixture-of-uniforms}
  Suppose that $P \in \mathcal{P}(\R)$ has properly disconnected support. Then, we have
  \begin{equation*}
    \mathcal{S}_{\textup{SL}}(P, P) = \emptyset \, .
  \end{equation*}
\end{restatable}
The proof, illustrated in Figure~\ref{fig:non-existence-for-mixture}, has four main steps.
First, the straight-line assumption and the ODE~\eqref{eqn:def_flow_map_ode} imply that the flow consists of non-intersecting straight-lines.
Second, the topology of the domain and the disconnectedness of the modes $S_0$ and $S_1$ force flow lines that start at $S_i$ to stay in $S_i$ for $i \in \{0,1\}$.
Third, the second step implies that there is a space-time region $G$ such that $X_t \notin G$ at each $t$ almost surely; this region $G$ separates $S_0$ from $S_1$. Since $X_\bullet$ has continuous sample paths and $G$ has nonempty interior, the path $t \mapsto X_t$ never enters $G$ almost surely. We call $G$ a \emph{no-go zone}.
Fourth, by independence, it follows that a pair $(X_0, X_1) \in S_0 \times S_1$ is drawn with positive probability. Thus, a continuous path $t \mapsto X_t$ with endpoints in $S_0$ and $S_1$ is drawn with positive probability. This path intersects $G$, which separates $S_0$ from $S_1$, yielding a contradiction.

\begin{figure}[ht!]
\centering
\begin{tikzpicture}[x=7cm,y=3cm,xscale=0.55,yscale=0.55]

  \draw[axis] (0,-1.4) -- (0,1.4);
  \draw[axis] (-0.05,-1.3) -- (1.05,-1.3);
  \node[note,anchor=north] at (0.52,-1.3) {time $t$};
  \node[note,anchor=east]  at (0,0) {space $x$};

  \pulseRightAt{0}{-0.85}{0.22}{0.045}
  \pulseRightAt{0}{+0.85}{0.22}{0.045}
  \node[note,anchor=east] at (-0.04,-0.85) {$S_0$};
  \node[note,anchor=east] at (-0.04,+0.85) {$S_1 $};

  \pulseLeftAt{1}{-0.85}{0.22}{0.045}
  \pulseLeftAt{1}{+0.85}{0.22}{0.045}
  \node[note,anchor=west] at (1.05,-0.85) {$S_0$};
  \node[note,anchor=west] at (1.05,+0.85) {$S_1$};

  \foreach \y in {-1.00,-0.92,-0.84,-0.76,-0.68}{
    \draw[streamline] (0,\y) -- (1,\y);
  }
  \foreach \y in {0.68,0.76,0.84,0.92,1.00}{
    \draw[streamline] (0,\y) -- (1,\y);
  }

  \node[note,red!20!black,anchor=west] at (0.0,0.22) {No--Go Zone};
  \draw[red!70!black,fill=red!70!black,fill opacity=0.5] (0.0,0.63) rectangle (1.0, -0.63);

  \draw[domain=0:1,smooth,variable=\x,bluepath] plot ({\x},{-1 + 2/(1+exp(-8 * (\x-0.5)))});

  \fill[cyan!60!blue] (0.00,-0.964) circle (1.6pt);
  \fill[cyan!60!blue] (1.00,0.964) circle (1.6pt);

\end{tikzpicture} \begin{tikzpicture}[x=7cm,y=3cm,xscale=0.55,yscale=0.55]

  \draw[axis] (0,-1.4) -- (0,1.4);
  \draw[axis] (-0.05,-1.3) -- (1.05,-1.3);
  \node[note,anchor=north] at (0.52,-1.3) {time $t$};
  \node[note,anchor=east]  at (0,0.5) {space $x$};

  \filledBellAtLeft{0}{0}{0.30}{0.045}
  \node[note,anchor=east] at (-0.04,-0.10) {$\mathcal{N}(0,1)$};

  \filledBellAtRight{1}{-1}{0.10}{0.06}
  \filledBellAtRight{1}{+1}{0.10}{0.06}
  \node[note,anchor=west] at (1.05,-1.0) {$S_0$};
  \node[note,anchor=west] at (1.05,1.0) {$S_1$};

  \foreach \y in {-0.4, -0.30,-0.20,-0.10,0.00}{
    \draw[streamline] (0,\y-0.1) -- (1,\y-0.80);
  }
  \foreach \y in {0.00, 0.10,0.20,0.30, 0.4}{
    \draw[streamline] (0,\y+0.1) -- (1,\y+0.80);
  }

  \node[note,red!20!black,anchor=west] at (0.5,-0.25)
    {Low--Go Zone};

  \path[fill=red!70!black, fill opacity=0.5] (0.00,-0.10) -- (0.00,0.10) -- (1, 0.80) -- (1, -0.80) -- cycle;

  \draw[bluepath]
    (0.00,-0.15)
    .. controls (0.30,-0.50) and (0.45,-0.10) ..
    (0.55,0.05)
    .. controls (0.68,0.35) and (0.82,0.75) ..
    (1.00,0.94);

  \fill[cyan!60!blue] (0.00,-0.15) circle (1.6pt);
  \fill[cyan!60!blue] (1.00,0.94) circle (1.6pt);

\end{tikzpicture} \caption{
  Depiction of the proofs for Theorems~\ref{thm:non-existence-for-mixture-of-uniforms} (Left) and~\ref{thm:non-existence-for-Gaussian-mixture} (Right). The blue curves show paths that occur with positive probability which contradict the existence of the no- and low-go zones.
}
\label{fig:non-existence-for-mixture}
\end{figure}

At this stage, one might wonder whether the impossibility is due to selecting mixture models for both endpoints. In Section~\ref{subsec:connected-support} we will see that this is not the case. \ymm{First, we present a $d\geq 1$-dimensional version of the previous theorem.}

\begin{restatable}{theorem}{ThmNonExistHighD}
  \label{thm:non-existence-mixture-of-uniforms-in-high-d}
  Suppose $P \in \mathcal{P}(\Rd)$ has disconnected support on the sets $S_0$ and $S_1$ (cf.\ Definition~\ref{def:disconnected-support}) and further suppose that
  \begin{equation*}
    P(S_0) \neq P(S_1) \, .
  \end{equation*}
  Then, we have
  \begin{equation*}
    \mathcal{S}_{\textup{SL}}(P, P) = \emptyset.
  \end{equation*}
\end{restatable}
While we intuit that the result in Theorem~\ref{thm:non-existence-mixture-of-uniforms-in-high-d} holds for equal probability $P(S_0) = P(S_1)$, this case is currently not covered by the above argument.
Finally, we remark that decomposing the measure into exactly two disconnected components $S_0, S_1$ satisfying the assumptions of Definition~\ref{def:disconnected-support} is not essential. In fact, one can easily generalize the above arguments to a measure $P$ supported on any $N \in \N$ components $S_1, \ldots, S_N$ for both the $d=1$ and $d > 1$ cases. One imposes the natural assumptions, e.g., each $S_i$ is open and bounded, for each $i \neq j$ we have $\closedconv S_i \cap \closedconv S_j = \emptyset$. In the $d > 1$ case, we have the additional condition that $P(S_i) \neq P(S_j)$ for $i \neq j$.

\subsection{Connected support in \texorpdfstring{$d=1$}{d=1}}
\label{subsec:connected-support}

While our results above cover the setting of exactly disconnected support of our target,
we can indeed get a similar impossibility result for a target measure with $\epsilon$-disconnected support.
\begin{definition}
  \label{def:epsilon-disconnected-support}
  A measure $P \in \mathcal{P}(\Rd)$ has \emph{$\epsilon$-disconnected support} if there exist bounded open sets $S_0, S_1 \subseteq \Rd$ such that $\closedconv S_0 \cap \closedconv S_1 = \emptyset$, $P(S_i) > 0$ for $i = 0,1$, and
  \begin{equation*}
    P(S_0) + P(S_1) \geq 1 - \epsilon \, .
  \end{equation*}
  Furthermore, we say $P$ has a \emph{properly} $\epsilon$-disconnected support if, in addition to the above, the following positivity condition holds: for each $i \in \{0,1\}$ and any Borel set $A \subseteq S_i$ with positive Lebesgue measure, we have $P(A) > 0$.
\end{definition}

\begin{definition}
  The \emph{modulus of continuity} of a function $f : [0,1] \to \R$ is defined as
  \begin{equation*}
    \kappa_f(\delta) = \sup_{\substack{t, s \in [0,1] \\ |t-s| \leq \delta}} |f(t) - f(s)| \, .
  \end{equation*}
  The modulus of continuity of the stochastic process $X_\bullet: [0,1] \to \R$ is
  the random function $\kappa_{X_\bullet} : (0,1] \to [0, \infty)$ with realizations
  \begin{equation*}
    \kappa_{X_\bullet}(\delta; \omega) = \sup_{\substack{t, s \in [0,1] \\ |t-s| \leq \delta}} |X_t(\omega) - X_s(\omega)|,\,\, \omega\in\Omega.
  \end{equation*}
\end{definition}
We write random quantites $\kappa_{X_\bullet}(\delta)$ for $\omega \mapsto \kappa_{X_\bullet}(\delta; \omega)$ and $\kappa_{X_\bullet}$ for ${\omega \mapsto (\delta \mapsto \kappa_{X_\bullet}(\delta; \omega))}$.

\begin{definition}
  \label{def:concentration-class}
  For positive reals $A, \alpha, \beta > 0$, let $\mathcal{C}(A, \alpha, \beta)$ denote the class of stochastic processes $X_\bullet$ with continuous sample paths satisfying
  \begin{equation*}
    \mathbb{P}\left( \kappa_{X_\bullet}(\delta) \geq \theta \right) \leq A \frac{\delta^\alpha}{\theta^\beta}
  \end{equation*}
\end{definition}
for all $\theta, \delta > 0$. We illustrate processes belonging to $\mathcal{C}(A, \alpha, \beta)$ in Example~\ref{ex:concentration-of-interpolants}.

The modulus of continuity $\kappa_{X_\bullet}$ allows us to promote marginal control on $X_\bullet$ to sample path control. For example, concentration inequalities on $\kappa_{X_\bullet}$ (cf.~Definition~\ref{def:concentration-class}) control the number of crossings of $X_\bullet$ through a low probability region, as in Lemma~\ref{lemma:modulus-of-continuity-control}. This strategy is essential below.

The technology we have developed so far allows us to answer a question posed in the end of the last subsection: namely, was the result of Theorem~\ref{thm:non-existence-for-mixture-of-uniforms} an artifact of the disjoint supports of the target measure? The answer is \emph{no}, as one can see from the following result.

\begin{restatable}{proposition}{PropNonExistGaussian}
  \label{prop:non-existence-for-Gaussian-mixture-to-Gaussian-mixture}
  For all $A, \alpha, \beta > 0$ there exist $\mu, \sigma > 0$ such that the mixture model
  \begin{equation*}
    P = \frac{1}{2} \, \mathcal{N}(\mu, \sigma) + \frac{1}{2} \, \mathcal{N}(-\mu, \sigma)
  \end{equation*}
  satisfies
  \begin{equation*}
    \mathcal{S}_{\textup{SL}}(P, P) \cap \mathcal{C}(A, \alpha, \beta) = \emptyset \, .
  \end{equation*}
\end{restatable}
This result is a quantitative sharpening of Theorem~\ref{thm:non-existence-for-mixture-of-uniforms}. One proceeds in the same way to obtain space-time region $G$ separating $S_0$ from $S_1$ such that the probability that $X_t$ enters $G$ is $\epsilon$, with $\epsilon$ vanishing as $\sigma \to 0$. We call $G$ a \emph{low-go zone}.
Then, we use the concentration properties of a processes in $\mathcal{C}(A, \alpha, \beta)$ to show that the probability that $X_\bullet$ crosses the region $G$ vanishes with $\epsilon$ which in turn vanishes with $\sigma$. But, similar to the proof of Theorem~\ref{thm:non-existence-for-mixture-of-uniforms}, the path $t \mapsto X_t$ must cross $G$ with positive probability independent of $\sigma$, yielding a contradiction for $\sigma$ small enough. \ymm{(Note that this result also follows as a corollary of Theorem~\ref{thm:non-existence-for-Gaussian-mixture} below.)}

We now generalize Theorem~\ref{thm:non-existence-for-mixture-of-uniforms} in two ways.
First, the target measure is allowed to be supported everywhere, as in Proposition~\ref{prop:non-existence-for-Gaussian-mixture-to-Gaussian-mixture}. Second, the source measure now includes all positive absolutely continuous measures satisfying some mild regularity assumptions. In particular, the result holds for a standard normal source measure,  a common choice in generative flows.
Before stating and proving the result, we make the needed regularity of the source measure precise.
\begin{definition}
  A measure $P \in \mathcal{P}(\Rd)$ is \emph{positive absolutely continuous} if it is absolutely continuous with respect to the Lebesgue measure and if, for any Borel set $A \subseteq \R$ with positive Lebesgue measure, we have $P(A) > 0$.
\end{definition}
\begin{definition}
  A measure $P \in \mathcal{P}(\Rd)$ is said to be \emph{$(C, \gamma)$-Frostman} for some $C, \gamma > 0$ if we have
  \begin{equation*}
    P\big( B_r(x) \big) \leq C \, r^\gamma
  \end{equation*}
  for all $x \in \Rd$ and $r > 0$, where $B_r(x) = \{ y \in \Rd : \|y - x\| < r \}$ is an open ball.
\end{definition}
Measures with $L^p$ densities for $p \in (1, \infty]$ are $(C, \gamma)$-Frostman for some $C, \gamma > 0$; see Example~\ref{ex:Lp-are-Frostman}.

\begin{restatable}{theorem}{ThmNonExistGaussianMixture}
  \label{thm:non-existence-for-Gaussian-mixture}
  Let $P_0 \in \mathcal{P}(\R)$ be any positive absolutely continuous measure that is $(C, \gamma)$-Frostman for some $C, \gamma > 0$.
  Moreover, let $P_1$ be any absolutely continuous measure with properly $\epsilon$-disconnected support for some $\epsilon > 0$ and positive weights $w_i \coloneq P_1(S_i) > 0$ for $i = 0,1$.
  Then, for any real numbers $A, \alpha, \beta > 0$ there exists an $\epsilon_0 = \epsilon_0\big(C, \gamma, S_0, S_1, w_0, w_1, A, \alpha, \beta\big)$ depending only on the numbers $C, \gamma$, the support geometries $S_0, S_1$, the real numbers $w_0, w_1$, and the numbers $A, \alpha, \beta$ such that, if $\epsilon \leq \epsilon_0$, then
  \begin{equation*}
    \mathcal{S}_{\textup{SL}}(P_0, P_1) \cap \mathcal{C}(A, \alpha, \beta) = \emptyset \, .
  \end{equation*}
\end{restatable}

The proof of Theorem~\ref{thm:non-existence-for-Gaussian-mixture}, illustrated in Figure~\ref{fig:non-existence-for-mixture}, extends the key ideas of Theorem~\ref{thm:non-existence-for-mixture-of-uniforms} and Proposition~\ref{prop:non-existence-for-Gaussian-mixture-to-Gaussian-mixture}.
It has five main steps. The first two steps mirror the beginning of the proof of Theorem~\ref{thm:non-existence-for-mixture-of-uniforms} and establish that the flow consists of non-intersecting straight lines.
We then show that the flow lines started in each mode $S_i$ stay in $S_i$ for each $i\in\{0,1\}$.
Since the support is no longer disconnected, we use a stronger topological argument that only holds in dimension $d=1$.
In the third step, we obtain another low-go zone $G$, now of trapezoidal shape, such that $X_t \in G$ for each time $t \in [0,1]$ with some small probability $\epsilon > 0$.
In the fourth step, we deduce that the probability the process $X_\bullet$ crosses through $G$ is bounded by an expression that is vanishing in $\epsilon$. The trapezoidal geometry plays a key role here; see Remark~\ref{rk:trapezoidal-space-time-geometry} for further discussion.
In the final step, we follow the argument of Theorem~\ref{thm:non-existence-for-mixture-of-uniforms} to deduce that $X_\bullet$ crosses $G$ with probability at least $w_0 \, w_1 > 0,$ where $w_0$, $w_1$ are the weights of the modes. For $\epsilon$ small enough, we arrive at a contradiction.

\begin{restatable}{corollary}{CorNormalSourceNonExist}
  Fix $P_0 = \mathcal{N}(0,1)$ and weights $w_0, w_1 > 0$ with $w_0 + w_1 = 1$ and positive reals $A, \alpha, \beta > 0$ defining the class $\mathcal{C}(A, \alpha, \beta)$.
  There exist $ \mu_i \in \R$ and $ \sigma_i > 0$ for $i \in \{0,1\}$, such that
  \begin{equation*}
    P_1 = w_0 \, \mathcal{N}( \mu_0, \sigma_0) + w_1 \, \mathcal{N}( \mu_1, \sigma_1)
  \end{equation*}
  satisfies
  \begin{equation*}
    \mathcal{S}_{\textup{SL}}(P_0, P_1) \cap \mathcal{C}(A, \alpha, \beta) = \emptyset \, .
  \end{equation*}
\end{restatable}

This section has elucidated a fundamental obstruction to constructing straight-line processes with independent endpoints connecting a regular source measure $P_0$ with a multi-modal target measure $P_1$. 
The obstruction rests on the intrinsic tension between the rigid space-time geometry of straight-line flows and the independent endpoint assumption.
Empirically, such phenomena have been widely observed but, to our knowledge, never analyzed.

\section{Conclusion and discussion}

This work builds a structural theory of one-shot generative flows. Specifically, we study the existence of stochastic processes with independent endpoints that yield straight-line ODE flow maps. From a practical perspective, this is tantamount to asking whether one can enjoy inference-time optimal generative models with minimal upfront computation.

We first develop PDEs that link the Lagrangian notion of straightness to Eulerian characteristics of the velocity field and its moments; these can be understood as balance laws that relate the covariance tensor of the velocity to the ensemble acceleration field. One implication of this characterization is that affine processes yielding straight-line flows must have deterministically coupled endpoints. We then use this characterization constructively, introducing new Gaussian flows that explicitly demonstrate the \textit{feasibility} of performing one-shot generation while keeping the endpoints of the underlying process independent. On the other hand, we also prove \textit{impossibility} theorems that reveal core limitations of the modern generative flow paradigm: in canonical multi-modal settings, straightness is fundamentally incompatible with endpoint independence. Our obstruction results rely on a core geometric insight regarding the ``rigid'' space-time geometry of straight-line flows---a geometry that processes with independent endpoints cannot realize. We believe the analytical techniques used in these results (Section~\ref{sec:impossibility}) may also be of interest beyond this paper.

\paragraph{Open questions.}
Several natural directions remain open:
\begin{enumerate}
    \item \emph{Extend the impossibility theory for  measures with connected support to dimensions $d > 1$.} 
The more refined analytical techniques developed in Theorem~\ref{thm:non-existence-for-Gaussian-mixture} could allow for extensions to the case of $\epsilon$-separated supports in $\mathbb{R}^d$ for $d >1$.
The core difficulty is 
in obtaining a low-go zone (as in Theorem~\ref{thm:non-existence-for-Gaussian-mixture}) that separates high-dimensional measures.

    \item \emph{Construct straight-line processes beyond the Gaussian setting.} Gaussian endpoint measures allow the closed-form expression of the Eulerian quantities in Section~\ref{sec:PDE_characterization}. Given the impossibility results for multi-modal distributions (Section~\ref{sec:impossibility}), we wish to examine the construction of straight-line flows for generic unimodal endpoint marginals, e.g., log-concave measures.

  \item \emph{Understand intermediate regimes between independence and determinism.} The present work studies processes with independent endpoints, while other constructions discussed in the introduction have sought to approximate a deterministic endpoint coupling. We wish to study the intermediate regime of partial endpoint coupling, thereby illuminating a more continuous tradeoff between pre-computation and inference-time efficiency.

\end{enumerate}

\begin{ack}
PT and YMM acknowledge support from the US Air Force Office of Scientific Research (AFOSR) MURI program, under award number FA9550-20-1-0397, and from the ExxonMobil Technology and Engineering Company. PT also acknowledges support from the Onassis Foundation PhD Fellowship. DS acknowledges support from the 2025--2026 MathWorks Fellowship. DS and YMM acknowledge support from the US Department of Energy, Office of Science, Office of Advanced Scientific Computing Research, via the M2dt MMICC center under award number DE-SC0023187 and through the Scientific Discovery through Advanced Computing (SciDAC) FASTMath Institute, under contract number DE-AC52-07NA27344. 

PT also thanks Professor Andre Wibisono for helpful discussions that inspired this line of work.
\end{ack}

\newpage
\medskip
{
    \small
    \bibliographystyle{abbrv}

}

\appendix
\section{Setup and Notation}\label{appendix:notation}
\subsection{Probability Theory}
\label{subsec:notation}
Let $(\Omega, \mathcal{F}, \mathbb{P})$ be an abstract probability space.
Random variables are measurable functions $X: \Omega \to \Rd$,
\footnote{In this paper, subsets of $\Rd$ are always equipped with the Borel $\sigma$-algebra.}
and the expectation $\mathbb{E}[X]$ is defined as the integral $\int_\Omega X \, d\mathbb{P}$.
For random variables $X, Y: \Omega \to \Rd$ the conditional expectation $\mathbb{E}\left[ X \mid Y \right]$ is defined as the conditional expectation
$\mathbb{E} \left[ X \mid \sigma(Y) \right]$ where $\sigma(Y)$ is the $\sigma$-algebra generated by $Y$ and for $y \in \R$ the conditional expectation
$\mathbb{E} \left[ X \mid Y=y \right]$ is the integral $\int_\Omega X \, d\mathbb{P}_y$ where $\mathbb{P}_y$ is the \emph{disintegration}~\cite[Section 10.6]{bogachev2007measure} of $\mathbb{P}$ on the level sets of $Y$.
Finally, given another measurable space $(\Omega', \mathcal{F}')$ we denote the set of measures on $\Omega'$ by $\mathcal{P}(\Omega')$ and the subset of absolutely continuous measures by $\mathcal{P}_{\textup{a.c.}}(\Omega')$.

\subsection{Analysis}
\label{subsec:notation-analysis}

Let $C^k([0,1]; \Rd)$ be the space of component-wise $k$-times continuously differentiable functions from $[0,1]$ to $\Rd$ and $W^{k,p}([0,1]; \Rd)$ be the space of component-wise $k$-times weakly differentiable functions $[0,1] \to \Rd$ with weak (component-wise) derivatives in $L^p([0,1])$.
An $\Rd$-valued stochastic process is a collection $\{X_t\}_{t \in I}$ indexed by some set $I$, where $X: \Omega \to \Rd$ is measurable.
In this paper, we take $I = [0,1]$ and write $X \coloneq (X_t)_t \coloneq (X_t)_{t \in [0,1]}$.
A sample path of a stochastic process is the function $t \mapsto X_t(\omega)$ for a fixed realization $\omega \in \Omega$.
We note that a stochastic process with sample paths in the Sobolev space $W^{k,p} \coloneq W^{k,p}([0,1]; \Rd)$ can equivalently be viewed as a random variable $X: \Omega \to W^{k,p}$ and its law is thus in $\mathcal{P}(W^{k,p})$.

A family of measures $\{P_t\}_{t \in [0,1]}$ is \emph{narrowly continuous}, if for every continuous and bounded function $\varphi \in C_b(\Rd)$ the map $t \mapsto \int_{\Rd} \varphi(x) \, dP_t(x)$ is continuous.

For real numbers $a, b \in \R$ we denote their minimum and maximum by
\begin{equation*}
    a \wedge b \coloneq \min(a,b) \, , \quad a \vee b \coloneq \max(a,b) \, .
\end{equation*}

For a set $S \subseteq \Rd$ its topological closure is denoted by $\overline{S}$ and defined as the smallest closed set containing $S$:
\begin{equation*}
    \overline{S} \coloneq \bigcap_{\substack{S \subseteq C \subseteq \Rd \\ C \textup{ closed}}} C \, .
\end{equation*}
The interior of $S$ is denoted by $S^\circ$ and defined as the largest open set contained in $S$:
\begin{equation*}
    S^\circ \coloneq \bigcup_{\substack{C \subseteq S \subseteq \Rd \\ C \textup{ open}}} C \, ,
\end{equation*}
and the boundary of $S$ is defined as $\partial S \coloneq \overline{S} \setminus S^\circ$.
The convex hull of $S$ is denoted by $\textup{conv} \, S$ and defined as the smallest convex set containing $S$:
\begin{equation*}
    \textup{conv} \, S \coloneq \bigcap_{\substack{S \subseteq C \subseteq \Rd \\ C \textup{ convex}}} C \, .
\end{equation*}
The closed convex hull of $S$ is denoted by $\overline{\textup{conv}} \, S$ and defined as the smallest closed convex set containing $S$:
\begin{equation*}
    \overline{\textup{conv}} \, S \coloneq \bigcap_{\substack{S \subseteq C \subseteq \Rd \\ C \textup{ closed and convex}}} C \, .
\end{equation*}
For a finite subset $S = \{x_1, \ldots, x_n\}$ of $\Rd$ we write $|S|$ for the cardinality of $S$. For an infinite subset $S \subseteq \Rd$ we write $|S|$ for the Lebesgue measure of $S$. The distinction will be clear from context.

For any two sets $A, B \subseteq \Rd$ we define their distance as
\begin{equation*}
    \| A - B \| \coloneq \inf_{a \in A, b \in B} \| a - b \| \, ,
\end{equation*}
where $\| \cdot \|$ is the Euclidean norm on $\Rd$. For $x \in \Rd$, if $A = \{x\}$ is a singleton we mildly abuse notation and write $\| x - B\| \coloneq \inf_{b \in B} \| x - b \|$.

\subsection{Dynamics \& Linear Algebra}
\label{subsec:notation-dynamics_linalg}
Fix a stochastic process $X_\bullet \coloneq (X_t)_{t\in [0,1]}$ with sample paths in the Sobolev space $W^{2,1}([0,1]; \mathbb{R}^d)$.
We define the \emph{conditional velocity} and \emph{conditional acceleration} fields, also referred to as the \emph{ensemble velocity} and \emph{ensemble acceleration}, by
\[
  v(t,x) \coloneq \mathbb{E} \left[ \, \dot X_t \mid X_t = x \right] \quad \text{and} \quad a(t,x) \coloneq \mathbb{E} \left[ \, \ddot X_t \mid X_t = x \right] \, .
\]
We define the \emph{second moment velocity tensor} and the \emph{covariance} or \emph{Reynolds stress tensor} as
\[
  \mathcal{C}(t,x) \coloneq \mathbb{E} \left[ \, \dot X_t \otimes \dot X_t \mid X_t = x \right] \quad \text{and} \quad \Pi(t,x) \coloneq \mathcal{C}(t,x) - v(t,x) \otimes v(t,x) \, ,
\]
respectively.
We let $\mu_t = \textup{Law}(X_t)$ be the marginal law of $X$ at time $t$ and write $\rho_t$ for the density of $\mu_t$ with respect to the Lebesgue measure, if it exists.

We define the \emph{material derivative} of velocity $v$ at $x \in \Rd$ as
\[
    D_t \, v(t,x) \coloneq \partial_t \, v(t,x) + \left(v(t,x) \cdot \nabla\right) \, v(t,x).
\]
Because we often consider the space $\Rd$ over the time domain $[0,1]$, we often denote the time-dependent velocity vector field as $v_t$, defined as the mapping $x \mapsto v(t,x)$, with similar definitions for $a_t,\,\mathcal{C}_t,\,\Pi_t$.

Finally, for matrices $A, B \in \R^{d_1 \times d_2}$ we denote the \emph{Frobenius inner product} by
\[
  A : B \coloneq \textup{Tr}(A^\top B) = \sum_{i=1}^{d_1} \sum_{j=1}^{d_2} A_{ij} B_{ij},
\]
and for a matrix field $V: \Rd \to \R^{d \times d}$ consisting of differentiable entries we define its \emph{divergence} as the vector field $\nabla \cdot V: \Rd\to \Rd$ with components
$(\nabla \cdot V)_j (x) = \sum_{i=1}^d \partial_i V_{ij}(x)$.
Lastly, we define a \emph{positive semi-definite} matrix $A \in \R^{d \times d}_{+}$ as a symmetric matrix such that $x^\top A x \geq 0$ for all $x \in \Rd$. We call such a matrix \emph{positive definite} and denote $A \in \R^{d \times d}_{++}$ if in fact we have the strict inequality $x^\top A x > 0$ for all nonzero $x \in \Rd$.
 \subsection{Visualization setup for Figure~\ref{fig:straight-lines-gauss}}
For the one-dimensional case, we connect endpoints $P_0 = \mathcal{N}(-2, 0.6^2)$ and $P_1 = \mathcal{N}(3, 1.5^2)$. For the three-dimensional case, set $G(\theta)$ as the Given's rotation matrix
\[
  G(\theta) = \begin{bmatrix}
    \cos(\theta) & -\sin(\theta)\\
    \sin(\theta) & \cos(\theta)
  \end{bmatrix}.
\]
Then, set the endpoint measures $P_0 = \mathcal{N}(m_0,\Sigma_0)$ and $P_1 = \mathcal{N}(m_1,\Sigma_1)$ such that
\begin{gather*}
  m_0 = \begin{bmatrix} -7\\-5\\2.5\end{bmatrix},\quad V_0 = \begin{bmatrix}G(\theta^*) & 0\\0 & 1\end{bmatrix},\quad \Lambda_0 = \mathrm{diag}(0.6, 3, 0.8)^2,\quad \Sigma_0 = V_0\Lambda_0 V_0^\top,\\
  m_1 = -m_0,\quad V_1 = \begin{bmatrix}1 & 0\\0 & G(\theta^*)\end{bmatrix},\quad \Lambda_1 = \mathrm{diag}(0.8, 0.6, 3)^2,\quad \Sigma_1 = V_1\Lambda_1 V_1^\top,
\end{gather*}
where $\theta^* = \frac{2\pi}{3}$. We remark that this three-dimensional construction gives $\Sigma_0,\,\Sigma_1$ as non-commutative matrices, suggesting that some amount of rotation is handled by the flow-map, as opposed to simply scaling in each principle direction.

\section{Proofs for Section~\ref{sec:PDE_characterization}}\label{subsec:proofs-PDE_characterization}

\RegAssumption*

\PropReformulation*
\begin{proof}
    It is clear that $(2) \implies (1)$.
    Let us show that $(1) \implies (2)$.
    We have
    \begin{equation*}
        \frac{d^2}{dt^2} \phi_t(x) = 0 \quad \implies \quad \exists c \in \Rd \, : \, \partial_t \phi_t(x) = c \, ,
    \end{equation*}
    so using the boundary condition $\phi_0(x) = x$ we get
    \begin{equation*}
        \phi_t(x) = x + t \, c \, .
    \end{equation*}
    This further implies
    \begin{equation*}
        c = \phi_1(x) - x  \, ,
    \end{equation*}
    and doing some algebra we get
    \begin{equation*}
        \phi_t(x) = (1-t) \, x + t \, \phi_1(x) \, .
    \end{equation*}

    Finally, let us show that $(1) \iff (3)$. By definition we have
    \begin{equation*}
        \partial_t \phi_t(x) = v_t(\phi_t(x)) \, ,
    \end{equation*}
    and, therefore, by the chain rule
    \begin{align*}
        \frac{d^2}{dt^2} \phi_t(x) &= \partial_t v_t(\phi_t(x)) \\
        &= \partial_t v_t(\phi_t(x)) + \left( v_t(x) \cdot \nabla \right) v_t(\phi_t(x)) \\
        &= D_{t} v_t(\phi_t(x)) \, .
    \end{align*}
    Since $\phi_t: x \mapsto \phi_t(x)$ is a homeomorphism, it is in particular bijective, so we have
    \begin{equation*}
        \frac{d^2}{dt^2} \phi_t(x) = 0 \quad \iff \quad D_{t} v_t(x) = 0 \, ,
    \end{equation*}
    for all $x \in \Rd$.
\end{proof}

\begin{restatable}{lemma}{LemmaMomentumBalance}
  \label{lemma:momentum-balance}
  Consider the velocity $v_t$, second moment $\SecMomTens$, and acceleration $a_t$ of a process $X_t$ with marginal density $\rho_t$. Then, for all $(t,x) \in [0,1] \times \Rd$, we have
  \begin{equation} \label{eq:momentum}
    \partial_{t} \,(\rho_t \, v_t) + \nabla \cdot (\rho_t \, \SecMomTens) = \rho_t \, a_t \, .
  \end{equation}
\end{restatable}
\begin{proof}
  For a vector valued test function $\Phi \in C^\infty_c(\Rd; \Rd)$, we have
  \begin{align*}
    \int  \Phi(x) \cdot \partial_t(\rho_t(x) \, v_t(x)) \, dx &= \partial_t \int \Phi(x) \cdot (\rho_t(x) \, v_t(x)) \, dx \\
    &= \partial_t \int \Phi(x) \cdot v_t(x) \, d\rho_t(x) \\
    &= \partial_t \mathbb{E} \left[ \, \Phi(X_t) \cdot \mathbb{E} \left[ \dot X_t \mid X_t \right] \, \right] \\
    &= \partial_t \mathbb{E} \left[ \, \Phi(X_t) \cdot \dot X_t \, \right] \\
    &= \mathbb{E} \left[ \, \Phi(X_t) \cdot \ddot X_t \, \right] + \mathbb{E} \left[ \, \left( \nabla \Phi(X_t) \, \dot X_t \right) \cdot \dot X_t \, \right] \\
    &= \mathbb{E} \left[ \, \Phi(X_t) \cdot a_t(X_t) \, \right] + \mathbb{E} \left[ \, \nabla \Phi(X_t) : \dot X_t \otimes \dot X_t \, \right] \\
    &= \mathbb{E} \left[ \, \Phi(X_t) \cdot a_t(X_t) \, \right] + \mathbb{E} \left[ \, \nabla \Phi(X_t) : \SecMomTens(X_t) \right] \\
    &= \int \Phi(x) \cdot \left(\rho_t(x) \, a_t(x)\right) \, dx + \int \nabla \Phi(x) : \left( \rho_t(x) \, \SecMomTens(x) \right) \, dx \\
    &= \int \Phi(x) \cdot \left(\rho_t(x) \, a_t(x)\right) \, dx - \int \Phi(x) \cdot \left[ \nabla \cdot \left( \rho_t(x) \, \SecMomTens(x) \right) \right] \, dx \\
  \end{align*}
  where we have only used definitions and the integration-by-parts formula.
  Since the above holds for all $\Phi \in C^\infty_c(\Rd; \Rd)$, we can conclude that
  \begin{equation*}
    \partial_t(\rho_t \, v_t) + \nabla \cdot (\rho_t \, \SecMomTens) = \rho_t \, a_t \, ,
  \end{equation*}
  concluding the proof.
\end{proof}

\LemmaReynolds*
\begin{proof}
    By the definition of the Reynolds stress tensor we have
    \begin{equation*}
      \SecMomTens(x) = v_t \otimes v_t + \Pi_t(X_t) \, .
    \end{equation*}
  Now write $\pi^{(i)}_t \in \Rd$ for the $i$-th row of $\Pi_t$ and $v^i_t, a^i_t \in \R$ for the $i$-th components of $v_t$ and $a_t$, respectively.
  Plugging the above display into~\eqref{eq:momentum} and looking at the $i$-th component of the resulting vector we obtain
  \begin{equation*}
    \partial_t\left( \rho_t \, v_t^i \right) + \nabla \cdot \left( \rho_t \, v^i_t \, v_t  \right) + \nabla \cdot \left( \rho_t \, \pi^{(i)}_t \right)  = \rho_t \, a^i_t \, .
  \end{equation*}
  Now expanding the differential operators we have
  \begin{equation*}
    v^i_t \, \partial_t \rho_t + \rho_t \, \partial_t v^i_t + v^i_t \, \nabla \cdot (\rho_t \, v_t) + \rho_t \, v_t \cdot \nabla v^i_t + \nabla \cdot (\rho_t \, \pi^{(i)}_t) = \rho_t \, a^i_t \, .
  \end{equation*}
  and rearranging
  \begin{equation*}
    v^i_t \, \Big( \partial_t \rho_t + \nabla \cdot (\rho_t \, v_t) \Big) + \rho_t \, \Big( \partial_t v^i_t + v_t \cdot \nabla v^i_t \Big) + \nabla \cdot (\rho_t \, \pi^{(i)}_t) = \rho_t \, a^i_t \, .
  \end{equation*}
  Using the continuity equation~\eqref{eq:continuity_equation}, the first parenthesis vanishes. Vectorizing the equation we obtain
  \begin{equation*}
    \rho_t \, \left( \partial_t v_t + v_t \cdot \nabla v_t \right) + \nabla \cdot (\rho_t \, \Pi_t) = \rho_t \, a_t \, .
  \end{equation*}
  Finally, using the definition of the material derivative we conclude.
\end{proof}

\ThmAffineProcesses*
\begin{proof}
  Fix $R > 0$ and consider a cut-off function $\eta_R \in C^\infty_c(\Rd)$ such that $\supp(\eta_R) \subseteq B_{2R}(0)$ and $\eta_R(x) = 1$ on the closure of $B_R(0)$. Moreover, we can construct $\eta_R$ such that $\eta_R(x) = \eta(x / R)$ for some $\eta \in C^\infty_c(\Rd)$ with  $\supp(\eta) \subseteq B_2(0)$, $\eta(x) = 1$ on the closure of $B_1(0)$ and $0 \leq \eta \leq 1$.

  Now consider the test function $\Phi_R(x) \in C^\infty_c(\Rd; \Rd)$ given by $\Phi_R(x) = \eta_R(x) \, x$ and compute
  \begin{align*}
    \int \Phi_R(x) \cdot \left( \nabla \cdot \left(\rho_t \, \Pi_t \right) \right) \, \dd x &=-\int \nabla \Phi_R(x) : \left( \rho_t \, \Pi_t \right) \, \dd x \\
    &=- \int \nabla \Phi_R(x) : \Pi_t(x) \, \rho_t(x) \, \dd x \\
    &=- \int_{B_R(0)} I_d : \Pi_t(x) \, \rho_t(x) \, \dd x - \int_{B_R(0)^c} \nabla \Phi_R(x) : \Pi_t(x) \, \rho_t(x) \, \dd x \\
    &=- \int_{B_R(0)} \textup{Tr} \, \Pi_t(x) \, \rho_t(x) \, \dd x - \int_{B_R(0)^c} \nabla \Phi_R(x) : \Pi_t(x) \, \rho_t(x) \, \dd x \\
  \end{align*}
  Taking the limit $R \to \infty$ and using Lemma~\ref{lemma:integrability} together with dominated convergence we obtain
  \begin{equation*}
    \int_{\Rd} \textup{Tr} \, \Pi_t(x) \,\rho_t(x) \, \dd x = 0
  \end{equation*}
  which is equivalent to
  \begin{equation*}
    \sum_{i=1}^d \E \left[ \textup{Var}\left( \, \dot X_t^i \, | \, X_t \right) \right] = 0 \, .
  \end{equation*}
  Now since $\textup{Var}(\dot X_t^i | X_t) \geq 0$ almost surely we have that
  \begin{equation*}
    \textup{Var}\left( \, \dot X_t^i \, | \, X_t \, \right) = 0 \, ,
  \end{equation*}
  almost surely for all $i \in \{1, \ldots, d\}$ and all $t \in [0,1]$.
  Thus, there are Borel measurable maps $G_t: \Rd \to \Rd$ such that each $t \in [0,1]$
  \begin{equation*}
    \dot X_t = G_t(X_t) \text{ a.s.}
  \end{equation*}
  Taking $t = 0$ this reads $X_1 = X_0 + G_0(X_0)$ almost surely.
  Setting $T = \textup{id} + G_0$ we can conclude
  \begin{equation*}
    X_1 = T(X_0) \text{ a.s.}
  \end{equation*}
  This concludes the proof of the first part of the theorem.
  For the second part, we note since we assumed the form $X_t = (1-t) X_0 + t X_1$ then the deterministic map
  \begin{equation*}
    T_t(x) \coloneq (1-t) \, x + t \, T(x) \, ,
  \end{equation*}
  satisfies
  \begin{equation*}
    X_t = T_t(X_0) \text{ a.s.}
  \end{equation*}
  Now observe that for any $x \neq y \in \Rd$ we have
  \begin{equation*}
    (T_t(x) - T_t(y)) \cdot (x-y) = (1-t) \, \| x-y \|^2 + t \, (T(x) - T(y)) \cdot (x-y) \, .
  \end{equation*}
  By the fundamental theorem of calculus and the fact that the Jacobian of $T$ is positive semi-definite we have
  \begin{equation*}
    \left(T(x) - T(y) \right) \cdot (x-y) = \int_0^1 (x-y) \cdot \nabla T(s \, x + (1-s) \, y) \cdot (x-y) \, ds \geq 0 \, ,
  \end{equation*}
  and thus
  \begin{equation*}
    (T_t(x) - T_t(y)) \cdot (x-y) > 0 \implies T_t(x) \neq T_t(y) \;\; \text{for} \;\; x \neq y \, .
  \end{equation*}
  In other words, for any $t \in [0,1)$ the map $T_t(\cdot): \Rd \to \Rd$ has an inverse $T_t^{-1}: \Rd \to \Rd$ onto its image and we can write
  \begin{equation*}
    X_0 = T_t^{-1}(X_t) \text{ a.s.} \quad \textup{for} \;\; t \in [0,1) \, .
  \end{equation*}
  In fact, the assumed injectivity of $T$ allows us to extend the above display to the closed interval $t \in [0,1]$.
  It follows that
  \begin{equation*}
    \dot X_t = X_1 - X_0 = T(X_0) - X_0 = T(T_t^{-1}(X_t)) - T_t^{-1}(X_t) \text{ a.s.}
  \end{equation*}
  Therefore, $\dot X_t$ is a measurable function of $X_t$ which immediately implies that
  \begin{equation*}
    \Pi_t(X_t) = \textup{Var}(\dot X_t | X_t) = 0 \textup{ a.s.} \textup{ for all } t \in [0,1] \, .
  \end{equation*}
Since $\rho_t$ is the density of $X_t$ we have that
  \begin{equation*}
    \Pi_t(x) = 0 \textup{ for $\rho_t$-a.e. $x \in \Rd$ for all } t \in [0,1] \, ,
  \end{equation*}
  and so
  \begin{equation*}
    \nabla \cdot \left( \rho_t \, \Pi_t \right) = 0 \textup{ for all } t \in [0,1] \, .
  \end{equation*}
  This completes the proof.
\end{proof}
\begin{remark}\label{rk:OT-map-gives-straight-lines}
  Note that for the optimal transport map $T$ with the squared Euclidean cost between absolutely continuous distributions there is a convex potential $\phi: \Rd \to \R$ such that $T = \nabla \phi$, see~\cite[Theorem 2.5.10]{figalli2021invitation}. Further assuming that $T$ is continuously differentiable, we thus obtain that the Jacobian of $T$ will be positive semi-definite at any $x \in \Rd$.
\end{remark}

\begin{lemma}
  \label{lemma:integrability}
With notation as in the proof of Theorem~\ref{thm:affine_processes}, there is a constant $C > 0$ depending only the function $\eta$ such that
\begin{equation*}
    \left| \nabla \Phi_R(X_t) : \Pi_t(X_t) \right| \leq C \, \left\| \Pi_t(X_t) \right\|_{F_1} \, .
  \end{equation*}
  with $\| \cdot \|_{F_1}$ denoting the Frobenius $1$-norm defined by $\| A \|_{F_1} = \sum_{ij} |A_{ij}|$ for any matrix $A$.
  Moreover, 
we have
  \begin{equation*}
    \E \, \left\| \Pi_t(X_t) \right\|_{F_1} < \infty \;\; \textup{for all } t \in [0,1] \, .
  \end{equation*}
\end{lemma}
\begin{proof}
  With notation as in the proof of Theorem~\ref{thm:affine_processes}, note that the cut-off function $\eta_R$ satisfies
  \begin{equation}\label{eq:cut-off-Leibniz}
    \nabla \Phi_R(x) = \eta_R(x) \, I_d + x \otimes \nabla \eta_R(x) \, .
  \end{equation}
  By construction, we have
  \begin{equation}\label{eq:cutoff-estimate-1}
    \sup_{x \in \Rd} | \eta_R(x) | \leq 1 \, ,
  \end{equation}
  \begin{equation*}
    \nabla \eta_R(x) = \frac{1}{R} \, \nabla \eta(x / R) \, .
  \end{equation*}
  Note that $\nabla \eta$ is a non-zero only on the closed annulus $\overline {B_2(0)} \setminus B_1(0)$ and by the smoothness of $\eta$ we we have $\sup_{x \in B_2(0) \setminus B_1(0)} \| \nabla \eta(x) \| = C < \infty$. Thus, we have the estimate
  \begin{equation}\label{eq:cutoff-estimate-2}
    \sup_{x \in \Rd} \max_{i,j} \left| x_i \,  \partial_j \eta_R(x) \right| \leq 2 \, C
  \end{equation}
  By combining equations~\eqref{eq:cut-off-Leibniz},~\eqref{eq:cutoff-estimate-1} and~\eqref{eq:cutoff-estimate-2} we can estimate
  \begin{equation*}
    \sup_{x \in \Rd} \max_{i,j} \left| \nabla \Phi_R^{ij}(x) \right| \leq 1 + 2 \, C
  \end{equation*}
  and thus for any $t \in [0,1]$ we have
  \begin{equation*}
    \left| \nabla \Phi_R(X_t) : \Pi_t(X_t) \right|  \leq \sup_{x \in \Rd} \max_{i,j} \left| \nabla \Phi_R^{ij}(X_t) \right| \sum_{k\ell} \left| \Pi_t^{k\ell}(X_t) \right| \leq \left(1 + 2 \, C \right) \, \left\| \Pi_t(X_t) \right\|_{F_1} \, ,
  \end{equation*}
  almost surely. 
  
  Here, note that due to the assumptions in Theorem~\ref{thm:affine_processes} we have $X_t = (1-t) X_0 + t X_1$ and by Assumption~\ref{assumptoion:regularity} we can conclude
\begin{equation}\label{eq:dot-Xt-is-L2}
    \E \, \| \dot X_t \|^2 \leq 2 \, (1-t)^2 \, \E \, \| X_0 \|^2 + 2 \, t^2 \, \E \, \| X_1 \|^2 < \infty \, .
  \end{equation}
  
  Finally, we note that for each $i, j \in \{1, \ldots, d\}$ we use the conditional Jensen inequality to write
  \begin{equation*}
    \E \left[ \left| \Pi_t^{ij}(X_t) \right| \right] \leq \E \left[ \left| \dot X_t^i \, \dot X_t^j \right| \right] + \E \left[ \left| v_t^i(X_t) \, v_t^j(X_t) \right| \right] 
  \end{equation*}
  Now we can use Cauchy-Schwartz to further estimate
  \begin{align*}
    \E \left[ \left| \dot X_t^i \, \dot X_t^j \right| \right] &\leq \sqrt{\E \left[ \left| \dot X_t^i \right|^2 \right] \, \E \left[ \left| \dot X_t^j \right|^2 \right]} < \infty \, , \\
    \E \left[ \left| v_t^i(X_t) \, v_t^j(X_t) \right| \right] &\leq \sqrt{\E \left[ \left| v_t^i(X_t) \right|^2 \right] \, \E \left[ \left| v_t^j(X_t) \right|^2 \right]} < \infty \, ,
  \end{align*}
  where the first line follows by \eqref{eq:dot-Xt-is-L2} and the second follows again by the Jensen inequality
  \begin{equation*}
    \E \left[ \left| v_t^i(X_t) \right|^2 \right] = \E \left[ \left| \E \left[ \dot X_t^i | X_t \right] \right|^2 \right] \leq \E \left[ \E \left[ \left| \dot X_t^i \right|^2 | X_t \right] \right] =  \E\left[ |\dot X_t^i|^2\right] < \infty \, .
  \end{equation*}
  Thus, we conclude
  \begin{equation*}
    \E \left[ \left\| \Pi_t(X_t) \right\|_{F_1} \right] = \E \left[ \sum_{ij} \left| \Pi_t^{ij}(X_t) \right| \right] < \infty \, ,
  \end{equation*}
  as required.
\end{proof}

\begin{lemma}[Continuity Equation]
  \label{lemma:continuity_equation}
  We have the identity
  \begin{equation*}\label{eq:continuity_equation}
    \partial_t \,  \rho_t + \nabla \cdot \left( \rho_t \, v_t \right) = 0
  \end{equation*}
  for all $t \in [0,1]$.
\end{lemma}
\begin{proof}
  Fix a test function $\varphi \in C^\infty_c(\Rd)$ and compute
  \begin{align*}
    \int_{\Rd} \varphi(x) \, \partial_t \, \rho_t(x) \, dx
    &= \partial_t \, \int_{\Rd} \varphi(x) \, d \rho_t(x) \\
    &= \partial_t \, \mathbb{E} \left[ \varphi(X_t) \right] \\
    &= \mathbb{E} \left[ \nabla \varphi(X_t) \cdot \dot X_t \right] \\
    &= \mathbb{E} \Big[ \nabla \varphi(X_t) \cdot v_t(X_t) \Big] \\
    &= \int_{\Rd} \nabla \varphi(x) \cdot v_t(x) \, d\rho_t(x) \\
    &= -\int_{\Rd} \varphi(x) \, \nabla \cdot \left( \rho_t(x) \, v_t(x) \right) \, dx
  \end{align*}
  where we used integration-by-parts, the fact that $\varphi$ has compact support as well as properties of the conditional expectation.
  Since the above holds for all $\varphi \in C^\infty_c(\Rd)$ we get the desired result.
\end{proof}
 
\section{Proofs for Section~\ref{sec:existence}}\label{subsec:proofs-existence}
\begin{lemma}
  \label{lemma:mean-squared-diff-of-Sigma_t}
  Recall that a stochastic process $X_\bullet: t \mapsto X_t$ with $X_t \in L^2(\Omega; \Rd)$ for all $t \in [0,1]$ is said to be \emph{mean-square differentiable} if there exists a process $\dot X_\bullet: t \mapsto \dot X_t$ with $\dot X_t \in L^2(\Omega; \Rd)$ for all $t \in [0,1]$ such that
  \begin{equation*}
    \lim_{h \to 0} \E \left[ \left\| \frac{X_{t+h} - X_t}{h} - \dot X_t \right\|^2 \right] = 0 \, ,
  \end{equation*}
  for all $t \in [0,1]$.
  For any such process, the covariance matrix
  \begin{equation*}
    \Sigma_t \coloneq \textup{Cov}(X_t, X_t) \, ,
  \end{equation*}
  is differentiable in a mean-square sense and satisfies
  \begin{equation*}
    \dot \Sigma_t = \textup{Cov}(\dot X_t, X_t) + \textup{Cov}(X_t, \dot X_t) \, .
  \end{equation*}
\end{lemma}
\begin{proof}
  First, recall that
  \begin{equation*}
    \textup{Cov}(X_t, X_t) = \E \left[ X_t X_t^\top \right] - \E[X_t] \, \E[X_t]^\top \, .
  \end{equation*}
  Examining the two terms separately, we start by defining
  \begin{equation*}
    \Delta_h X_t \coloneq \frac{X_{t+h} - X_t}{h}
  \end{equation*}
  for $h \neq 0$ sufficiently small and recalling that by the mean-square differentiability assumption we have
  \begin{equation*}
    \Delta_h X_t \xrightarrow{L^2} \dot X_t \textup{ as } h \to 0 \, .
  \end{equation*}
  Now by the Cauchy-Schwartz inequality this implies
  \begin{equation*}
    \Delta_h X_t \to \dot X_t \textup{ in } L^1 \textup{ as } h \to 0 \, ,
  \end{equation*}
  and thus
  \begin{equation*}
    \frac{d}{dt} \left( \E[X_t] \, \E[X_t]^\top \right) = \E[\dot X_t] \, \E[X_t]^\top + \E[X_t] \, \E[\dot X_t]^\top \, .
  \end{equation*}
  Now for the second term we have
  \begin{align*}
    \Delta_h \left( \E[X_t X_t^\top] \right) &= \E \left[ \Delta_h \left( X_t X_t^\top \right) \right] \\
    &= \frac{1}{h} \E \left[ X_{t+h} \, X_{t+h}^\top - X_t \, X_t^\top \right] \\
    &= \frac{1}{h} \E \left[ X_{t+h} \, (X_{t+h} - X_t)^\top + (X_{t+h} - X_t) \, X_t^\top \right] \\
    &= \E \left[ X_{t+h} \, (\Delta_h X_t)^\top + (\Delta_h X_t) \, X_t^\top \right] \, ,
  \end{align*}
  and using the Cauchy--Schwartz inequality once more we have that
  \begin{align*}
    & \left( \Delta_h X_t \right) \, X_t^\top \xrightarrow{L^1} \dot X_t \, X_t^\top \textup{ as } h \to 0 \, , \\
    & X_{t+h} \, (\Delta_h X_t)^\top \xrightarrow{L^1} X_t \, \dot X_t^\top \textup{ as } h \to 0 \, ,
  \end{align*}
  and in the second line we have used that a mean-square differentiable process is also $L^2$-continuous.
  Indeed, one has
  \begin{align*}
      \E\left[ \| X_{t+h} - X_t \|^2 \right] &= h^2 \, \E \left[ \left\| \frac{X_{t+h} - X_t}{h} - \dot X_t + \dot X_t \right\|^2 \right] \\
      &\leq 2 \, h^2 \E \left[ \left\| \frac{X_{t+h} - X_t}{h} - \dot X_t \right\|^2 \right] + 2 \, h^2 \, \E \left[ \| \dot X_t \|^2 \right] \\
    \end{align*}
  and the right-hand side vanishes as $h \to 0$.
  Therefore, we obtain
  \begin{equation*}
    \frac{d}{dt} \left( \E[X_t X_t^\top] \right) = \E \left[ X_t \, \dot X_t^\top + \dot X_t \, X_t^\top \right] \, .
  \end{equation*}
  Combining the two terms above we conclude that
  \begin{align*}
    \dot \Sigma_t &= \frac{d}{dt} \left( \E[X_t X_t^\top] - \E[X_t] \, \E[X_t]^\top \right) \\
    &= \E \left[ X_t \, \dot X_t^\top + \dot X_t \, X_t^\top \right] - \E[\dot X_t] \, \E[X_t]^\top - \E[X_t] \, \E[\dot X_t]^\top \\
    &= \textup{Cov}(\dot X_t, X_t) + \textup{Cov}(X_t, \dot X_t) \, ,
  \end{align*}
  completing the proof.
\end{proof}

We first describe a straight-line generalized interpolant between two Gaussian measures with identical covariances.

\begin{restatable}{theorem}{ThmExistenceSameCov}\label{thm:existence-of-straight-line-interpolants-between-gaussians}
  Let $m_0, m_1 \in \Rd$ and $\Sigma \in \R^{d \times d}_{++}$ and consider Gaussian measures
  $P_0 = \mathcal{N}( m_0, \Sigma)$ and $P_1 = \mathcal{N}( m_1, \Sigma)$.
  Now set
  \begin{gather*}
    F(t, x, y, z) = (1 - t) \, x + t \, y + \sqrt{2 \, t \, (1- t)} \, z \, \quad \textup{and} \quad
    Q = \mathcal{N}(0, \Sigma) \, .
  \end{gather*}
  The generalized interpolant $(F, Q)$ satisfies
  \begin{equation*}
    (F, Q) \in \mathcal{F}_{\textup{SL}}(P_0, P_1) \, .
  \end{equation*}
\end{restatable}
\begin{proof}
  The generalized interpolant $(F, Q)$ induces the process
  \begin{equation*}
    X_t = (1 - t) \, X_0 + t \, X_1 + \sqrt{2 \, t \, (1 - t)} \, Z \, .
  \end{equation*}
  Since all random variables involved are Gaussian, we can compute
  \begin{equation*}
    v(t, X_t) = \E \left[ \dot X_t \, \big| \, X_t \right] = \E[\dot X_t] + \textup{Cov}(\dot X_t, X_t) \, \textup{Var}(X_t)^{-1} \, \left( X_t - \E[X_t] \right),
  \end{equation*}
  where by $\E[Z] = 0$ we have
  \begin{equation*}
    \E[\dot X_t] = m_1 - m_0
  \end{equation*}
  and by the independence of $X_0, X_1, Z$ we have
  \begin{align*}
    &\textup{Cov}(\dot X_t, X_t) = \textup{Cov}\left( - X_0 + X_1 + \frac{1 - 2t}{2 \, \sqrt{t \, (1 - t)}} \, Z \, , \, (1 - t) \, X_0 + t \, X_1 + \sqrt{t \, (1 - t)} \, Z \right) \\
    &= - (1 - t) \, \Sigma + t \, \Sigma + \frac{1 - 2t}{\cancel{\sqrt{2 \, t \, (1 - t)}}} \, \cancel{\sqrt{2 \, t \, (1 - t)}} \, \Sigma \\
    &= \left( 2t - 1 + 1 - 2 t \right) \Sigma \\
    &= 0 \, .
  \end{align*}
  Thus, we conclude that $v(t, x) = m_1 - m_0$
  for each $(t,x) \in [0,1] \times \Rd$ leading to the flow map $\phi(t, x) = x + t \, (m_1 - m_0)$.
  In particular, we have that $\partial_t^2 \phi \equiv 0$.
  Noting that $\phi$ and $v$ satisfy Assumption~\ref{assumptoion:regularity} and the process $X_\bullet$ has sample paths in $W^{1,1}([0,1];\Rd)$ a.s. we use Corollary~\ref{corollary:straightness} and Remark~\ref{remark:reduced_regularit} to conclude.
\end{proof}

\begin{corollary}
  \label{eq:F_SL_nonempty_for_gaussians}
  For any $m \in \Rd$ and $\Sigma \in \R^{d \times d}_{++}$, consider the measure $P = \mathcal{N}(m, \Sigma)$. Then, we have that $\mathcal{F}_{\textup{SL}}(P, P) \neq \emptyset$.
\end{corollary}
\begin{remark}
  At this stage, one might be tempted to generalize the above strategy to pairs of measures
  \begin{equation*}
    P_0 = \mathcal{N}(m_0, \Sigma_0) \, , \quad P_1 = \mathcal{N}(m_1, \Sigma_1) \, ,
  \end{equation*}
  with $m_i \in \Rd$ and distinct $\Sigma_i \in \R^{d \times d}_{++}$ for $i = 0,1$.
  Unfortunately, this is not immediately possible.
  Indeed, consider matrix valued schedules $(A_t)_t, (B_t)_t, (C_t)_t \in \R^{d \times d}$ satisfying the boundary conditions
  \begin{align*}
    A_0 &= I_d \, , \,\, A_1 = 0 \, , \\
    B_0 &= 0   \, , \,\, B_1 = I_d \, , \\
    C_0 &= 0   \, , \,\, C_1 = 0 \, .
  \end{align*}
  and draw $Z \sim \mathcal{N}(m_2, \Sigma_2)$ with $m_2 \in \Rd$ and $\Sigma_2 \in \R^{d \times d}_{++}$.
  Now set
  \begin{equation*}
    X_t = A_t \, X_0 + B_t \, X_1 + C_t \, Z \, ,
  \end{equation*}
  and, as done above, attempt to find $(A_t)_t, (B_t)_t, (C_t)_t$ such that $\textup{Cov}(\dot X_t, X_t) \equiv 0$.
  Define $\Sigma_t \coloneq \textup{Var}(X_t) = \textup{Cov}(X_t, X_t)$
  and by Lemma~\ref{lemma:mean-squared-diff-of-Sigma_t} note that we have $\dot \Sigma_t = \textup{Cov}(\dot X_t, X_t) + \textup{Cov}(X_t, \dot X_t)$. Since the covariance is symmetric, we obtain
  \begin{equation*}
    \dot \Sigma_t \equiv 0 \, .
  \end{equation*}

  In particular, this implies that $\Sigma_0 = \Sigma_1$ contradicting that the $\Sigma_i$ are distinct.
\end{remark}

\begin{restatable}{proposition}{PropOneDGaussians}\label{prop:straight-line-gaussians-1d}
    Let $P_0 = \mathcal{N}(m_0, \sigma_0^2)$ and $P_1 = \mathcal{N}(m_1, \sigma_1^2)$ be two univariate Gaussian measures with means $m_i \in \R$ and variances $\sigma_i^2 > 0$ for $i = 0,1$.
    Set
    \begin{align*}
        F(t, x, y, z) &= (1 - t) \, x + t \, y + \sqrt{2 \, t \, (1 - t)} \, z \, ,\\
        Q &= \mathcal{N}(0, \sigma_0 \sigma_1).
    \end{align*}
    Then, the generalized interpolant $(F, Q)$ satisfies
    \begin{equation*}
        (F, Q) \in \mathcal{F}_{\textup{SL}}(P_0, P_1) \, .
    \end{equation*}
\end{restatable}
\begin{proof}
    The generalized interpolant $(F, Q)$ induces the process
    \begin{equation}
        X_{t}= (1 - t) X_{0}+t X_{1} + \sqrt{2t(1-t)}Z,
    \end{equation}
    where $X_{0}\sim\mathcal{N}(m_{0},\sigma^{2}_{0}),\ X_{1}\sim\mathcal{N}(m_{1},\sigma_{1}^{2}),$ and $Z\sim\mathcal{N}(0,\sigma_{0}\sigma_{1})$ where all random variables are independent. We will verify that the velocity field $v$ associated to $X_\bullet$ satisfies Burgers' equation 
    \begin{equation*}
      D_t v_t \coloneq \partial_{t}v(t,x) + \left( v \cdot \nabla \right) v(t,x) = 0 \, .
    \end{equation*}
    First, using Gaussianity we compute the mean and variance of $X_t$ as
    \begin{align*}
        m_{t} &\coloneq \mathbb{E}[X_{t}] \\
              &= m_{0} + t(m_{1} - m_{0}) \, , \\
        \sigma_{t}^{2} &\coloneq \mathbb{V}\mathrm{ar}(X_{t}) \\
        &= \sigma_{0}^{2}(1-t)^{2}+\sigma_{1}^{2}t^{2} + 2t(1-t)\sigma_{0}\sigma_{1} \\
        &= (\sigma_{0}(1-t)+\sigma_{1}t)^{2} \, ,
    \end{align*}
    as well as the covariance term
    \begin{align*}
    \mathrm{Cov}(\dot{X}_{t},X_{t}) &= \mathrm{Cov}(X_{1}-X_{0}+\frac{1-2t}{\sqrt{2t(1-t)}}Z,(1-t)X_{0}+tX_{1}+ \sqrt{2t(1-t)}Z)\\
    &= -(1-t)\sigma^{2}_{0}+t\sigma_{1}^{2}+ (1-2t)\sigma_{0}\sigma_{1} \, .
    \end{align*}
    Now we notice a key identity:
    \begin{align*}
    \partial_{t}(\sigma_{t}^{2}) &= \partial_{t}((\sigma_{0}(1-t) + \sigma_{1}t)^{2}) = 2(\sigma_{0}(1-t)+\sigma_{1}t)(\sigma_{1}-\sigma_{0})\\
    &= 2(-\sigma_{0}^{2}(1-t) + \sigma_{1}^{2}t+\sigma_{0}\sigma_{1}(1-t) - \sigma_{0}\sigma_{1}t)\\
    &= 2(-\sigma_{0}^{2}(1-t) + \sigma_{1}^{2}t+ (1-2t)\sigma_{0}\sigma_{1})\\
    &= 2\mathrm{Cov}(\dot{X}_{t},X_{t}) \, ,
    \end{align*}
    that yields a succinct expression for the velocity field $v$ as
    \begin{align*}
    v(t,x) &:= \mathbb{E}[\dot{X}_{t}|X_{t}=x]\\
    &= \mathbb{E}[\dot{X}_{t}] + \mathrm{Cov}(\dot{X}_{t},X_{t})\mathbb{V}\mathrm{ar}(X_{t})^{-1}(x-\mathbb{E}[X_{t}])\\
    &= \dot{m}_{t} + \frac{1}{2} \frac{\partial_{t}(\sigma_{t}^{2})}{\sigma_{t}^{2}}(x - m_{t})\\
    &= \dot{m}_{t} + \partial_{t}\ln\sigma_{t}(x-m_{t})\\
    \end{align*}
    Plugging into Burgers' equation, we have
    \begin{align*}
    \partial_{t}v(x,t) + (v\cdot\nabla)v(x,t) &= \partial_{t}v(x,t) + v(x,t)\partial_{x}v(x,t)\\
    &= \ddot{m}_{t}+\partial_{t}^{2}\ln\sigma_{t}(x-m_{t}) + \partial_{t}\ln\sigma_{t}(-\dot{m}_{t}) + (\dot{m}_{t}+\partial_{t}\ln\sigma_{t}(x-m_{t}))(\partial_{t}\ln\sigma_{t})\\
    &= \ddot{m}_{t}+\partial_{t}^{2}\ln\sigma_{t}(x-m_{t}) + (\partial_{t}\ln\sigma_{t})^{2}(x-m_{t})\\
    &= \ddot{m}_{t} + (\partial_{t}^{2}\ln\sigma_{t}+ (\partial_{t}\ln\sigma_{t})^{2})(x-m_{t}) \, .
    \end{align*}
    We can simplify the factor involving $\sigma_{t}$ in the above display as
    \begin{align*}
    \partial_{t}^{2}\ln\sigma_{t}+ (\partial_{t}\ln\sigma_{t})^{2}&= \partial_{t}\left( \frac{\dot\sigma_{t}}{\sigma_{t}}\right)+\left(\frac{\dot\sigma_{t}}{\sigma_{t}}\right)^{2}\\
    &= \frac{\ddot{\sigma_{t}}\sigma_{t}-(\dot\sigma_{t})^{2}}{\sigma_{t}^{2}} + \left(\frac{\dot\sigma_{t}}{\sigma_{t}}\right)^{2}\\
    &= \frac{\ddot\sigma_{t}}{\sigma_{t}}.
    \end{align*}
    Note here that since $\sigma_i > 0$ for $i \in \{0,1\}$ we have $\sigma_t > 0$ for each $t \in [0,1]$ and thus the above expression is well-defined.
    Now notice that $\ddot m(t) \equiv \ddot\sigma_{t} \equiv 0$. Indeed, recalling our computations above:
    \begin{align*}
        \ddot{m}_{t} &= \partial_{t}^{2}(m_{0}+t(m_{1}-m_{0})) = 0 \\
        \ddot\sigma_{t} &= \partial_{t}^{2}(\sigma_{0}(1-t) + \sigma_{1}t) = 0.
    \end{align*}
    Therefore, we substitute into the expression for Burgers' equation to conclude that
    \begin{align*}
    \partial_{t}v_{t}+(v\cdot\nabla)v &= \ddot{m}_{t}+(\partial^{2}_{t}\ln\sigma_{t}+(\partial_{t}\ln\sigma_{t})^{2})(x-m_{t})\\
    &=\ddot{m}_{t} + \frac{\ddot \sigma_t}{\sigma_t}(x-m_{t})\\
    &= 0.
    \end{align*}
    That is, we have $D_t v_t(x) \equiv 0$.
    Noting that $\phi$ and $v$ satisfy Assumption~\ref{assumptoion:regularity} and the process $X_\bullet$ has sample paths in $W^{1,1}([0,1];\Rd)$ a.s. we use Corollary~\ref{corollary:straightness} and Remark~\ref{remark:reduced_regularit} to conclude.
\end{proof}

\ThmMultivariateGaussians*
\begin{proof}
We remark that, because both covariance matrices are positive definite, the pencil $(\Sigma_0, \Sigma_1)$ is also positive definite---the pencil's eigenvalues are shared by the matrix $\Sigma_0^{-1/2}\Sigma_1\Sigma_0^{-1/2}$ (using the symmetric square root $\Sigma_0^{1/2}$), which is also SPD. Therefore, we can take $\Lambda$ as the diagonal matrix of (positive) eigenvalues of the pencil. Similarly, the pencil has nonsingular eigenvector matrix $\widetilde{V}$ defined such that $\Sigma_1 = \Sigma_0 \widetilde{V} \Lambda \widetilde{V}^{-1}$ and $\widetilde{V}^\top \Sigma_0 \widetilde{V} = I$. Thus, we know $V = \widetilde{V}^{-\top}$ satisfies the identity $\Sigma_0 = VV^\top$. We observe that $\Sigma_1$ has a similar property:
\[V^{-1}\Sigma_1 = \widetilde{V}^{\top}\Sigma_1 = \widetilde{V}^\top \Sigma_0 \widetilde{V}\Lambda \widetilde{V}^{-1} = \Lambda \widetilde{V}^{-1} = \Lambda V^{\top}.\]
Our diagonal matrix $\Lambda$ is thus positive and $V$ satisfies the constraints, making $\Sigma_Z$ symmetric positive definite. Now we define the interpolant
\[
X_t = (1-t)X_0 + tX_1 + \sqrt{2t(1-t)}Z,\quad X_i\sim \mathcal{N}(m_i, \Sigma_i),\ i=0,1,\ Z\sim\mathcal{N}(0, \Sigma_Z),\]
where the auxiliary covariance matrix $\Sigma_Z$ is defined as in the theorem statement. Defining matrix $\Xi = \Lambda^{1/2} - I$, we find that
\begin{align*}
    \CovarT &= (1-t)^2\Sigma_0 + t^2 \Sigma_1 + 2t(1-t)\Sigma_Z\\
    &=  V\left((1-t)^2 I + t^2 \Lambda + 2t(1-t)\Lambda^{1/2}\right)V^\top\\
    &= V((1-t)I + t\Lambda^{1/2})^2 V^\top\\
    &= V(I + t\Xi)^2 V^\top\\
    \dCovarT &= 2 V(I + t\Xi)\Xi V^\top\\
    \ddCovarT &= 2 V\Xi^2 V^\top.
\end{align*}
Further, one can find that
\begin{align*}
    \mathrm{Cov}(\dot X_t, X_t) &= \mathrm{Cov}\left(X_1 - X_0 + \frac{1-2t}{\sqrt{2t(1-t)}}Z,\ (1-t)X_0 + tX_1 + \sqrt{2t(1-t)}Z\right)\\
    &= -(1-t)\Sigma_0 + t\Sigma_1 + (1-2t)\Sigma_Z\\
    &= V(-(1-t) I + t\Lambda + (1-2t)\Lambda^{1/2})V^\top\\
    &= V(\Xi + t\Xi^2)V^\top\\
    &= V(I + t\Xi)\Xi V^\top\\
    &= \frac{1}{2} \dCovarT.
\end{align*}

Then, we express the conditional velocity field as
\begin{align*}
    v(t,x) &= \mathbb{E}[\dot X_t | X_t = x]\\
    &= \mathbb{E}[\dot{X}_t ] + \mathrm{Cov}(\dot{X}_t, X_t )\mathrm{Var}(X_t)^{-1} (x - \mathbb{E}[X_t])\\
    &= \dot{m}_t + \frac{1}{2} \dCovarT \CovarT^{-1}(x-m_t).
\end{align*}
With the intention of making Burgers' equation $D_t v_t = 0$ explicit in this example, we compute
\begin{align*}
    2\partial_t v &= 2\ddot{m}_t + \ddCovarT \CovarT^{-1}(x-m_t) - \dCovarT \CovarT^{-1}\dCovarT \CovarT^{-1}(x-m_t) + \dCovarT\CovarT^{-1} (-\dot{m}_t)\\
    &= 2\ddot{m}_t + (\ddCovarT + \dCovarT \CovarT^{-1}\dCovarT)\CovarT^{-1}(x-m_t) - \dCovarT\CovarT^{-1}\dot{m}_t.
\end{align*}
Similarly, we calculate $\nabla v = \frac12 \dCovarT\CovarT^{-1}$, giving
\[2(v\cdot\nabla)v = \dCovarT\CovarT^{-1}\dot m_t + \frac{1}{2}\dCovarT\CovarT^{-1}\dCovarT\CovarT^{-1}(x-m_t).\]
Using these above computations, we can write
\[2(\partial_t {v} + (v\cdot\nabla)v) = 2\ddot{m}_t + \left(\ddCovarT - \frac{1}{2}\dCovarT\CovarT^{-1}\dCovarT\right)\CovarT^{-1}(x-m_t).\]
We know, however, that mean $m_t$ has the form $m_t = (1-t)m_0 + tm_1$, making $\ddot{m}_t \equiv 0$ by construction. Further, we can expand the covariance term to see that
\begin{align*}
    \frac{1}{2}\dCovarT\CovarT^{-1}\dCovarT &= 2V(I + t\Xi)\Xi(I + t\Xi)^{-2}(I + t\Xi)\Xi V^\top\\
    &= 2V \Xi^2 V^\top\\
    &\equiv \ddCovarT,
\end{align*}
where the cancellations are due to the commutativity of diagonal matrices. Therefore, we can conclude $\ddCovarT - \frac{1}{2}\dCovarT\CovarT^{-1}\dCovarT\equiv 0$ which implies that
\begin{equation*}
  D_t v_t(x) = \partial_t v_t(x) + (v\cdot\nabla)v_t(x) \equiv 0,
\end{equation*}
Noting that $\phi$ and $v$ satisfy Assumption~\ref{assumptoion:regularity} and the process $X_\bullet$ has sample paths in $W^{1,1}([0,1];\Rd)$ a.s. we use Corollary~\ref{corollary:straightness} and Remark~\ref{remark:reduced_regularit} to conclude.
\end{proof}

\PropNonLipschitz*
\begin{proof}
    For the remainder of the proof, suppose $t\neq \tau$.
    First, consider the case $t < \tau$.
    We have
    \begin{align*}
        v_t(x) &= \mathbb{E}[\dot{X}_t | X_t = x]\\
        &= \mathbb{E}\left[-\frac{1}{\tau} X_0 \, \Big| \, \left(1-\frac{t}{\tau}\right)X_0 = x\right]\\
        &= -\frac{1}{\tau} \frac{\tau}{\tau - t} \mathbb{E}\left[ \frac{\tau-t}{\tau} X_0 \, \Big| \, \frac{\tau - t}{\tau} X_0 = x \right]\\
        &= \frac{x}{t - \tau}.
    \end{align*}
    using linearity in the third line and the definition of the conditional expectation in the fourth.
    Using similar reasoning, we can compute
    \begin{align*}
        \mathcal{C}_t(x) &\coloneq \mathbb{E}[\dot{X}_t \otimes \dot{X}_t|X_t = x]\\
        &= \frac{1}{\tau^2}\mathbb{E}\left[X_0 \otimes X_0 \, \Big| \left(1-\frac{t}{\tau}\right)X_0 = x \right]\\
        &= \frac{x\otimes x}{(t - \tau)^2} \\
        &= v_t(x)\otimes v_t(x) \, .
    \end{align*}
    Similarly, for the case $t > \tau$ we compute $v_t(x) = \frac{x}{t-\tau}$ and $\mathcal{C}_t(x) = v_t(x)\otimes v_t(x)$.
    Thus, for all $t\in[0,1]\setminus\{\tau\}$ we have
    \begin{equation*}
        \Pi_t(x) = \mathcal{C}_t(x) - v_t(x)\otimes v_t(x) \equiv 0.
    \end{equation*}
    On the other hand, the acceleration field $a_t(x) = \mathbb{E}[\ddot{X}_t | X_t = x] = 0$ is trivially vanishing.
    Therefore, the PDE~\eqref{eqn:reformulation_straightness_PDE_equiv} is trivially satisfied for all $t\in[0,1]\setminus\{\tau\}$.
    The velocity field $v_t(x) = \frac{x}{t - \tau}$, however, is not Lipschitz in $x$ uniformly in $t$. Indeed, for any $x \neq y \in \Rd$ note that
    \begin{equation*}
        \frac{\|v_t(x) - v_t(y)\|}{\|x-y\|} = \frac{1}{|t-\tau|},
    \end{equation*}
    and the right hand side diverges as $t \to \tau$.
\end{proof}
 
\section{Proofs for Section~\ref{sec:impossibility}}\label{subsec:proofs-impossibility}

\ThmNonExistMixUniforms*
\begin{proof}[Proof of Theorem~\ref{thm:non-existence-for-mixture-of-uniforms}]
  Let $S_0, S_1$ be as in Definition~\ref{def:disconnected-support} with $d=1$. By shifting the origin if necessary, we can assume without loss of generality that
  \begin{equation}
    \label{eq:origin-swift}
    \sup S_0 \eqcolon s_0 < 0 < s_1 \coloneq \inf S_1 \, .
  \end{equation}
  Suppose $\mathcal{S}_{\textup{SL}}(P, P) \neq \emptyset$ and take any $X_\bullet \in \mathcal{S}(P, P)$. We will work towards a contradiction.

  \subsubsection*{Step 1: Flow is continuous \& globally injective}
  Suppose that $v: [0,1] \times \R \to \R$ is the velocity field associated to the process $X_\bullet$, see Definition~\ref{def:straight-line-process-and-straight-line-interpolant}, and $\phi$ is the flow map driven by $v$. Now since
  \begin{equation*}
    \partial_t^2 \phi(t, x) = 0 \, ,
  \end{equation*}
  by Proposition~\ref{prop:reformulaton-1} we have that
  \begin{equation*}
    \phi(t, x) = x + t \, v(0, x) \, .
  \end{equation*}
  Because $v$ satisfies the assumptions of Definition~\ref{def:straight-line-process-and-straight-line-interpolant} then by \cite[Theorem 2.5]{teschl2012ordinary} we have
  \begin{equation}
    \label{eq:flow-is-non-intersecting-lines}
    x \neq x' \implies \phi(t, x) \neq \phi(t, x') \quad \textup{ for all } t \in [0,1] \, .
  \end{equation}
  Indeed, the initial value problem (IVP)
  \begin{equation*}
    \begin{cases}
      \frac{\dd}{\dd s} y(s) = v(s, y(s)) \\
      y(t) = z \, ,
    \end{cases}
  \end{equation*}
  admits unique local solutions for every $(t, z) \in [0,1] \times \R$. If $\phi(t, x) = \phi(t, x')$ for some $t_\star \in [0,1]$ and $x \neq x'$ then the curves
  \begin{equation*}
    y(t) = \phi(t, x) \, , \quad y'(t) = \phi(t, x') \, ,
  \end{equation*}
  are both local solutions to the above IVP with $t = t_\star$ and $z = \phi(t_\star, x) = \phi(t_\star, x')$ so they coincide in a neighborhood of $t_\star$. By a continuation argument, it must be  $y(t) = y'(t)$ for all $t\in [0,1]$ and so in particular $y(0) = y'(0)$ implying that $x = x'$, contradicting our assumption.

  Finally, let us note that, for each $t \in [0,1]$, the flow map $\phi(t, \cdot): \R \to \R$ is continuous. This follows from~\cite[Theorem 2.9]{teschl2012ordinary} and the assumed properties of $v$, see Definition~\ref{def:straight-line-process-and-straight-line-interpolant}. Note that the aforementioned theorem furnishes continuity of the two-point flow map that solves the above IVP from time $s$ to a later time $t$, which is combined with global existence on $[0,1]$, local uniqueness and a standard patching argument to obtain continuity of the map $x \mapsto \phi(t, x)$ for any $t \in [0,1]$.

  \subsubsection*{Step 2: Topological Obstruction}
  We claim that for  $P$-almost every $x \in S_i$ we have that
  \begin{equation}
    \label{eq:closure-of-modes-under-flow}
    \phi(t, S_i) \subseteq \closedconv S_i \textup{ for } i \in \{0,1\} \, .
  \end{equation}
  and all $t \in [0,1]$.
  In fact, since $S_i$ is open, $\phi(t, \cdot) : \R \to \R$ is a continuous map for each $t \in [0,1]$ (see \step{1}) and $\closedconv S_i$ is closed, it follows that the above display holds if and only if\footnote{
    For any open set $S \subset \Rd$, closed set $A \subseteq \Rd$ and continuous map $f: S \to \Rd$, if $f(x) \in A$ for Lebesgue a.e. $x \in S$ then $f(x) \in A$ for all $x \in S$. Indeed, the set $E = \{ x \in S : f(x) \notin A \}$ is open and has zero Lebesgue measure, so it must be empty.
  } it holds for all $x \in S_i$.
  Notice, here, that since the $\closedconv S_i$ are convex sets and by \step{1} the paths $t \mapsto \phi(t, x)$ are lines, for all $x$, it follows that
  \begin{equation*}
    \phi(t, S_i) \subseteq \closedconv S_i \iff \phi(1, S_i) \subseteq \closedconv S_i \, \quad \textup{for} \, i \in \{0,1\} \, .
  \end{equation*}
  Now by Corollary~\ref{corollary:generalized-interpolant-push-messure-forward} we have that
  \begin{equation*}
    \phi(1, \cdot)_\sharp P = P \, ,
  \end{equation*}
  and therefore it must be that
  \begin{equation}
    \label{eq:imgae-of-modes-under-flow}
    \phi(1, S_i) \subseteq \closedconv S_0 \cup \closedconv S_1 \textup{ for } i \in \{0,1\} \, .
  \end{equation}
  Indeed, if $x \in S_i$ satisfies $\phi(1, x) \notin \closedconv S_0 \cup \closedconv S_1$ then, by continuity of $x \mapsto \phi(1, x)$, there is a neighborhood $x \in U \subseteq S_i$ such that $\phi(1, U) \cap \left( \closedconv S_0 \cup \closedconv S_1 \right) = \emptyset$. Using the properness of the support of $P$, see Definition~\ref{def:disconnected-support}, we conclude that $\phi(1, \cdot)_\sharp P$ has support on a set disjoint from $\closedconv S_0 \cup \closedconv S_1$ and therefore it cannot be equal to $P$. Finally, recalling that the set $S_i$ is connected and the map $x \mapsto \phi(1, x)$ is continuous and using that the sets $\closedconv S_0$ and $\closedconv S_1$ are disjoint and by equation~\ref{eq:imgae-of-modes-under-flow} the set $\phi(1, S_i)$ lies in their union, it must be that $\phi(1, S_i) \subseteq \closedconv S_j$ for some $j \in \{0,1\}$, see~\cite[Lemma 23.2]{munkres2000topology}. By the global injectivity of the flow, it must be that $\phi(1, S_i) \subseteq \closedconv S_i$ for $i \in \{0,1\}$, proving the expression~\eqref{eq:closure-of-modes-under-flow}.

  \subsubsection*{Step 3: A no-go zone}
  Using~\eqref{eq:closure-of-modes-under-flow}, we have that
  \begin{equation*}
    \phi(t, S_i) \in \closedconv S_i \, \textup{ for } i \in \{0,1\} \, .
  \end{equation*}
  Hence, for all $x \in \Rd$, we have
  \begin{equation}
    \label{eq:no-go-zone}
    \phi(t, x) \notin G \coloneq \left(s_0, \, s_1 \right) \, .
  \end{equation}
    Here, recall that
  $\sup S_0 = s_0 < 0 < \inf S_1 = s_1$,\footnote{
    Here, we are using the following well-known representation: $x \in \closedconv S$ if and only if $x$ lies in the closure of ${\textup{conv} \, S}$ and $z \in \textup{conv} \, S$ if there exist points $y_1, \ldots, y_k \in S$ and weights $\lambda_1, \ldots, \lambda_k \in [0,1]$ with $\sum_{i=1}^k \lambda_i = 1$ such that $z = \sum_{i=1}^k \lambda_i y_i$. In particular, if $S$ is a subset of $\R$ then $\closedconv S \subseteq [\inf S, \sup S]$.
  } and we call $G$ the \emph{no-go zone}.
  Now, by Corollary~\ref{corollary:generalized-interpolant-push-messure-forward}, the process $X_\bullet \in \mathcal{S}(P_0, P_1)$ satisfies the distributional equality
  \begin{equation*}
    \phi(t, X_0) \overset{d}{=} X_t \, .
  \end{equation*}
  Combining this equivalence of distributions with~\eqref{eq:no-go-zone}, we have that
  \begin{equation}
    \label{eq:prob-1-at-each-t}
    \mathbb{P} \left( X_t \notin G \right) \textup{ for each } t \in [0,1] \, ,
  \end{equation}
  and intersecting over the rationals $\Q$ we get
  \begin{equation*}
    \mathbb{P}\left( \, \forall \, t \in \Q \cap [0,1], \, X_t \notin G \right) = 1 \, .
  \end{equation*}
  By continuity of the sample paths and openness of $G$, if $\exists \, \omega \in \Omega$ such that $X_t(\omega) \in G$ for some $t \in [0,1]$ then there exists an open set $t \in V \subset [0,1]$ such that $X_s(\omega) \in G$ for all $s \in V$. By density of $\Q$ in $\R$ there exists $t' \in V \cap \Q$ such that $X_{t'}(\omega) \in G$.
  It follows that
  \begin{equation*}
    \mathbb{P}\left( \, \exists \, t \in [0,1], \, X_t \in G \right) \leq \mathbb{P}\left( \, \exists \, t \in \Q \cap [0,1], \, X_t \in G \right) = 0 \, ,
  \end{equation*}
  which is equivalent to
  \begin{equation}
    \label{eq:prob-of-no-go-zone}
    \mathbb{P}\left( \, \forall \, t \in [0,1], \, X_t \notin G \right) = 1 \, .
  \end{equation}

  \subsubsection*{Step 4: Contradiction}
  Finally, note that with probability $P_1(S_0) \cdot P_1(S_1)$ we sample a pair $(X_0, X_1) \in S_0 \times S_1$.
  
  Now by continuity of the sample paths of $X_\bullet$ and the intermediate value theorem, there exists $t^* \in (0, 1)$ such that
  \begin{equation*}
    X_{t^*} \in G = \left(s_0, \, s_1 \right) \, .
  \end{equation*}
  In other words, we have shown that
  \begin{equation*}
    \mathbb{P}\left( \, \exists \, t \in [0,1], \, X_t \in G \right) \geq P_1(S_0) \cdot P_1(S_1) > 0 \, ,
  \end{equation*}
  contradicting equation~\eqref{eq:prob-of-no-go-zone} and completing the proof.
\end{proof}

\ThmNonExistHighD*
\begin{proof}
  \step{1} and \step{4} in the proof of Theorem~\ref{thm:non-existence-for-mixture-of-uniforms} extend verbatim.
  Therefore, it suffices to provide an alternate no-go zone $G \subseteq \Rd$ satisfying the following properties:
  \begin{enumerate}
    \item $G \subseteq \Rd$ is an open set,
    \item We have $\mathbb{P}(X_t \in G) = 0$ for all $t \in [0,1]$ and,
    \item For any (deterministic) path $\gamma \in C^0([0,1]; \Rd)$ such that $\gamma(0) \in S_0$ and $\gamma \in S_1$ there is a $0 < t < 1$ such that $\gamma(t) \in G$.
  \end{enumerate}
  We claim that the set
  \begin{equation*}
    G = \Rd \setminus \left( \closedconv S_0 \cup \closedconv S_1 \right) \, ,
  \end{equation*}
  satisfies the above properties. Indeed, since the sets $\closedconv S_0$ and $\closedconv S_1$ are closed the set $G$ is open. Moreover, assume any path $\gamma \in C^0([0,1]; \Rd)$ such that $\gamma(0) \in S_0$ and $\gamma(1) \in S_1$ does not enter $G$ at any $t \in [0,1]$. Define the last exit and first entry times of $\closedconv S_0$ and $\closedconv S_1$ respectively as
  \begin{align*}
    t_0 = \sup \{ t \in [0,1] : \gamma(t) \in \closedconv S_0 \} \, , \quad t_1 = \inf \{ t \in [0,1] : \gamma(t) \in \closedconv S_1 \} \, ,
  \end{align*}
  and note that $\gamma(t) \notin G$ for all $t$ implies that $t_0 = t_1$. But now for any $\epsilon > 0$ we have
  $\gamma(t_0 - \epsilon) \in \closedconv S_0$ and $\gamma(t_1 + \epsilon) \in \closedconv S_1$ therefore
  \begin{equation*}
    \left| \gamma(t_0 - \epsilon) - \gamma(t_1 + \epsilon) \right| \geq \textup{dist}(\closedconv S_0, \closedconv S_1) > 0 \, .
  \end{equation*}
  In other words,
  \begin{equation*}
    \lim_{t\to t_0+} \gamma(t) \neq \lim_{t \to t_0-} \gamma(t)
  \end{equation*}
  which contradicts the continuity of $\gamma$, completing the proof of property $3$.

  Finally, for property $2$ we argue as follows: recall that by \step{1} we have the straight-line representation of the flow map
  \begin{equation*}
    \phi(t, x) = (1-t) x + t \phi(1, x) \, .
  \end{equation*}
  Now, notice that we must have
  \begin{equation}
    \label{eq:phi-supported-on-convS0-union-convS1}
    \phi\left(1, S_0 \right) \subseteq \closedconv S_0 \cup \closedconv S_1 \, ,
  \end{equation}
  which follows from the following facts: $x \mapsto \phi(1, x)$ is continuous, \step{1} in the proof of Theorem~\ref{thm:non-existence-for-mixture-of-uniforms}, $\phi(1, \cdot)_\sharp P = P$ and $P$ is supported on $S_0 \cup S_1$, see \step{2} in the proof of Theorem~\ref{thm:non-existence-for-mixture-of-uniforms} for details. Moreover, again by the continuity of $\phi(1, \cdot): \Rd \to \Rd$ and the connectedness of each set $S_i$, it follows that the set $\phi(1, S_0)$ must be also connected, see~\cite[Lemma 23.2]{munkres2000topology}. Since $\closedconv S_0$ and $\closedconv S_1$ are disconnected and equation~\ref{eq:phi-supported-on-convS0-union-convS1} holds, it follows that
  \begin{equation*}
    \phi\left(1, S_0 \right) \subseteq \closedconv S_0  \quad \textup{or} \quad \phi\left(1, S_0 \right) \subseteq \closedconv S_1  \, .
  \end{equation*}
  The latter, however, is impossible since $P(S_0) \neq P(S_1)$ and $\phi(1, \cdot)_\sharp P = P$.
  Therefore, we have
  \begin{equation*}
    \phi(1, S_0) \subseteq \closedconv S_0 \, ,
  \end{equation*}
  which by convexity implies that
  \begin{equation*}
    \phi(t, S_0) \subseteq \closedconv S_0 \, ,
  \end{equation*}
  for all $t \in [0,1]$. Mutatis mutandis, we can show that for all $t \in [0,1]$ we have
  \begin{eqnarray}
    \phi(t, S_1) \subseteq \closedconv S_1 \, .
  \end{eqnarray}
  Finally, recalling the distributional equality
  \begin{equation*}
    X_t \overset{d}{=} \phi(t, X_0) \, , \quad X_0 \sim P
  \end{equation*}
  discussed in \step{3} of the proof of Theorem~\ref{thm:non-existence-for-mixture-of-uniforms} we obtain that
  \begin{equation*}
    \mathbb{P}(X_t \in G) = 0 \textup{ for all } t \in [0,1] \, ,
  \end{equation*}
  as $\phi(t,X_0)$ has no support on $G = \mathbb{R}^d \setminus (\closedconv S_0 \cup \closedconv S_1)$.
\end{proof}

\begin{definition}
  \label{def:up-crossing-number}
  For any real numbers $a < b$ and any stochastic process $X_\bullet : [0,1] \to \R$ with continuous sample paths we define the \emph{up-crossing number} of $X_\bullet$ with respect to the interval $(a,b)$ to be the random variable with realizations
  \begin{align*}
    N_{X_\bullet}(a,b;\omega) = \sup \Big\{
    & n \in \N : \exists \, t_1 < s_1 < t_2 < s_2 < \ldots < t_n < s_n \in [0,1] \\
    &\textup{ such that } X_{t_i}(\omega) \leq a < b \leq X_{s_i}(\omega) \textup{ for each } i = 1, \ldots, n \Big\} \, .
  \end{align*}
\end{definition}
Intuitively, this number $N_{X_\bullet}$ is the number of times a specific realization of the stochastic process crosses through the interval $[a,b]$.
\begin{remark}
  \label{rk:up-crossing-number-time-dependent}
  The above definition can easily be generalized to a time dependent interval $(a_t, b_t)$.
  Specifically, if $t \mapsto a_t$ and $t \mapsto b_t$ are continuous functions such that $a_t < b_t$ for all $t \in [0,1]$, then we can define $N_{X_\bullet}(a_\bullet, b_\bullet)$ by replacing the condition $X_{t_i}(\omega) \leq a < b \leq X_{s_i}(\omega)$ with $X_{t_i}(\omega) \leq a_{t_i} < b_{t_i} \leq X_{s_i}(\omega)$ in the above definition.
\end{remark}

\begin{restatable}{lemma}{LemmaModulusContinuity}
  \label{lemma:modulus-of-continuity-control}
  Let $ a < b$ be real numbers defining an interval $I \coloneq (a,b) \subseteq \R$ and let $X_\bullet$ be a stochastic process with continuous sample paths such that
  \begin{equation*}
    \mathbb{P}\left( X_t \in I \right) \leq \epsilon \textup{ for each } t \in [0,1] \, .
  \end{equation*}
  Moreover, suppose there exists $\alpha, \beta > 0$ such that, for all $\theta, \delta > 0$, we have the concentration inequality
  \begin{equation}
    \label{eq:modulus-of-continuity-concentration}
    \mathbb{P}\Big( \kappa_{X_\bullet}(\delta) \geq \theta \Big) \lesssim \frac{\delta^\alpha}{\theta^\beta}\, .
  \end{equation}
  Then, we have
  \begin{equation*}
    \mathbb{P}\Big( N_{X_\bullet}(a,b) \geq 1 \Big) \lesssim \left( \frac{\epsilon^\alpha}{(b-a)^\beta} \right)^{\frac{1}{\alpha + 1}} \, ,
  \end{equation*}
  where the constant depends on $\alpha$.
\end{restatable}
\begin{proof}
  Start by considering a mesh, i.e,  a finite subset $\Pi \subseteq [0,1]$ consisting of points $t_i = i \, \delta$ for $i = 0, \ldots, \lfloor 1 / \delta \rfloor$ with $t_{\lfloor 1 / \delta \rfloor + 1} = 1$, and note the event inclusion
  \begin{equation}
    \label{eq:mesh-event-inclusion}
    \Big\{ N_{X_\bullet}(a,b) \geq 1 \Big\} \subseteq \Big\{ \exists \, t_i \in \Pi \, : \, X_{t_i} \in I \Big\} \cup \Big\{ \kappa_{X_\bullet}(\delta) \geq b - a \Big\} \, .
  \end{equation}
  In other words, trajectories $X_\bullet$ that cross through $(a,b)$ at least once must either have at least one mesh point $t_i$ such that $X_{t_i} \in (a,b)$ or move though $(a,b)$ in a time interval smaller than the mesh size $\delta$.
  Indeed, fix a realization $\omega \in \Omega$ and suppose that $N_{X_\bullet}(a, b;\omega) \geq 1$. This implies that there exists $s < t \in [0,1]$ such that
  \begin{equation*}
    X_s(\omega) \leq a < b \leq X_t(\omega) \, .
  \end{equation*}
  Now, if there exists $t_i \in \Pi$ with $s \leq t_i \leq t$, then it follows that $X_{t_i}(\omega) \in I$ by continuity of the sample path $t \mapsto X_t(\omega)$. On the other hand, if no such $t_i$ exists then
  \begin{equation*}
    | s - t | \leq \delta \, .
  \end{equation*}
  In particular, the definition of the modulus $\kappa$ gives
  \begin{equation*}
    \kappa_{X_\bullet}(\delta;\omega) \geq |X_t(\omega) - X_s(\omega)| \geq b - a \, ,
  \end{equation*}
  which completes the proof of the event inclusion~\eqref{eq:mesh-event-inclusion}.

  Now applying a union bound to the right-hand side of equation~\eqref{eq:mesh-event-inclusion}, we obtain
  \begin{align*}
    \mathbb{P}\left( \, \exists \, t \in [0,1], \, X_t \in S \right) &\leq \sum_{t_i \in \Pi} \mathbb{P}\left( X_{t_i} \in S \right) + \mathbb{P}\left( \kappa_{X_\bullet}(\delta) \geq (b-a) \right) \\
    & \lesssim \frac{1}{\delta} \, \epsilon + \frac{\delta^\alpha}{(b-a)^\beta} \, .
  \end{align*}
  Minimizing the auxiliary function
  \begin{equation*}
    f(\delta) = \frac{1}{\delta} \, \epsilon + \frac{\delta^\alpha}{(b-a)^\beta}
  \end{equation*}
  we obtain the minimizer
  \begin{equation*}
    \delta^\star = \left( \frac{(b-a)^\beta}{\alpha} \, \epsilon \right)^{\frac{1}{\alpha + 1}} \, ,
  \end{equation*}
  with minimum value
  \begin{equation*}
    f(\delta^\star) = \left( \frac{\alpha + 1}{\alpha^{\frac{\alpha}{\alpha+1}}} \right) \, \frac{\epsilon^{\frac{\alpha}{\alpha + 1}}}{(b-a)^{\frac{\beta}{\alpha + 1}}} \, .
  \end{equation*}
  Plugging $\delta^\star$ in the above bound for the event $\{X_t \in S \}$ yields the desired result.
\end{proof}

\begin{remark}
  We note that choosing $[0,1]$ for the domain of the process $X_\bullet$ is immaterial and the proof would have gone through verbatim for any $[s_0, s_1]$ with $s_0 < s_1 \in \R$.
\end{remark}

\begin{restatable}{corollary}{CorTimeDependentUpcrossing}
  \label{cor:time-dependent-up-crossing-lemma}
  Let $a, b: [0,1]  \to \R$ be continuous functions such that $a_t < b_t$ for all $t \in [0,1]$, defining the time dependent interval $I_t = (a_t, b_t)$.
  Assume that there exists a real number $\Delta I > 0$ such that $b_t - a_t \geq \Delta I$ for all $t \in [0,1]$.
  Then, under the assumptions of Lemma~\ref{lemma:modulus-of-continuity-control} we have
  \begin{equation*}
    \mathbb{P}\Big( N_{X_\bullet}(a_\bullet, b_\bullet) \geq 1 \Big) \lesssim \left( \frac{\epsilon^\alpha}{\Delta I^\beta} \right)^{\frac{1}{\alpha + 1}} \, ,
  \end{equation*}
  with the constant depending on $\alpha$.
\end{restatable}
\begin{proof}
  The proof follows Lemma~\ref{lemma:modulus-of-continuity-control} mutatis mutandis.
\end{proof}

\PropNonExistGaussian*
\begin{remark}
  This result also follows as a Corollary to Theorem~\ref{thm:non-existence-for-Gaussian-mixture}.
  Below we sketch an alternate proof that is less cumbersome and functions as a natural intermediate step between the proofs of Theorem~\ref{thm:non-existence-for-mixture-of-uniforms} and Theorem~\ref{thm:non-existence-for-Gaussian-mixture}.  
\end{remark}
\begin{proof}[Proof sketch of Proposition~\ref{prop:non-existence-for-Gaussian-mixture-to-Gaussian-mixture}]
  With an appropriate modification of the second step of the proof of Theorem~\ref{thm:non-existence-for-mixture-of-uniforms} one can show that for any process $X_\bullet \in \mathcal{S}_{\textup{SL}}(P, P)$ there exists $a \in \R$ with $0 < a < \mu$ such that the interval $I = (-a,a)$ satisfies
  \begin{equation*}
    \mathbb{P}\left( X_t \in I \right) \leq \epsilon \,\, , \quad \forall \, t \in [0,1] \, ,
  \end{equation*}
  A direct application of Lemma~\ref{lemma:modulus-of-continuity-control} then yields that
  \begin{equation*}
    \mathbb{P}\Big( N_{X_\bullet}(-a,a) \geq 1 \Big) \leq A \left( \frac{\epsilon^\alpha}{(2a)^\beta} \right)^{\frac{1}{\alpha + 1}} \, ,
  \end{equation*}
  and the right-hand side can be made arbitrarily small by taking $\epsilon \searrow 0$ which is possible by taking $\sigma \searrow 0$ and keeping $\mu$ fixed. On the other hand, we have
  \begin{equation*}
    \mathbb{P}\Big( X_0 < - \mu \textup{ and } X_1 > \mu \Big) = \frac{1}{16} \, ,
  \end{equation*}
  and then by definition of an up-crossing we get
  \begin{equation*}
    \mathbb{P}\Big( N_{X_\bullet}(-a,a) \geq 1 \Big) \geq \mathbb{P}\Big( X_0 < - \mu \textup{ and } X_1 > \mu \Big) = \frac{1}{16} \, ,
  \end{equation*}
  yielding a contradiction and completing the proof.
\end{proof}

\ThmNonExistGaussianMixture*
\begin{proof}
Let $S_0, S_1$ be as in Definition~\ref{def:disconnected-support} with $d=1$. By shifting the origin if necessary, we can assume without loss of generality that
  \begin{equation}
    \label{eq:origin-swift-connect}
    \sup S_0 \eqcolon s_0 < 0 < s_1 \coloneq \inf S_1 \, .
  \end{equation}
  Moreover, sample $X\sim P_0$ and let $x_0, x_1 \in \R$ denote the unique real numbers satisfying
  \begin{equation*}
    \mathbb{P}\Big( X \leq x_0 \Big) = P_1(S_0) \quad \textup{and} \quad \mathbb{P}\Big( X \geq x_1 \Big) = P_1(S_1) \, .
  \end{equation*}
  
  We know these are unique by the construction of $P_0$ as strictly positive and absolutely continuous. 
  
  By the fact that $P_1(S_0) + P_1(S_1) \geq 1 - \epsilon > 0$ it follows that
  \begin{equation*}
    \mathbb{P} \Big( x_0 \leq X \leq x_1 \Big) \leq \epsilon \, .
  \end{equation*}
  Working towards a contradiction, suppose there exists
  \begin{equation*}
    X_\bullet \in \mathcal{S}_{\textup{SL}}(P_0, P_1) \cap \mathcal{C}(\alpha, \beta) \, .
  \end{equation*}

  First, note that \step{1} in the proof of Theorem~\ref{thm:non-existence-for-mixture-of-uniforms} extends verbatim. We proceed by modifying \step{2} to \step{4}.
  \subsubsection*{Step 2: Topological Obstruction}
  For $x, y \in \R$ define the space-time line-segment that connects the points $(0,x)$ and $(1,y)$ by
  \begin{equation*}
    \ell(x,y) \coloneq \Big\{ (t,z) \in [0,1] \times \R \, : \, z = (y - x) \, t + x \Big\} \, ,
  \end{equation*}
  Recall that the epigraph and hypograph of a function $f : D \subseteq \R \to \R$ are defined as
  \begin{equation*}
    \textup{epi}[f] = \{ (x,y) \in D \times \R : y \geq f(x) \} \quad\textup{and}\quad \textup{hypo}[f] = \{ (x,y) \in D \times \R : y \leq f(x) \} \, .
  \end{equation*}
  We can now construct the space-time half-spaces
  \begin{equation*}
    H_+ \coloneq \textup{hypo}\big[\ell(x_0, s_0)\big] \quad\textup{and}\quad H_- \coloneq \textup{epi}\big[\ell(x_1, s_1)\big] \, .
  \end{equation*}

  Now, consider the augmented flow map
  \begin{equation*}
    \Phi : [0,1] \times \R \to [0,1] \times \R \, , \quad \Phi(t,x) = (t, \phi(t,x)) \, .
  \end{equation*}
  If we define the half-lines
  \begin{equation*}
    I_+ \coloneq \{ x \in \R \, : \, x \geq x_1 \} \, , \quad I_- \coloneq \{ x \in \R \, : \, x \leq x_0 \} \, .
  \end{equation*}
  then the goal of this step is to show that for all $t \in [0,1]$ we have
  \begin{equation}
    \label{eq:I-to-H-seperation-mixture-of-Gaussians}
    \Phi(t, I_-) \subseteq H_- \, \quad\text{and}\quad \Phi(t, I_+) \subseteq H_+ \, .
  \end{equation}
  By the global injectivity of the flow, established in \step{1} of Theorem~\ref{thm:non-existence-for-mixture-of-uniforms} we observe that it suffices to show that
  
  \begin{equation}
    \label{eq:edge-point-control-mixture-to-mixture-d-1}
    \phi(1, x_i) = s_i \, , \quad i = 0,1 \, .
  \end{equation}
  Indeed, if~\eqref{eq:edge-point-control-mixture-to-mixture-d-1} is satisfied but~\eqref{eq:I-to-H-seperation-mixture-of-Gaussians} is violated then there exists, e.g., coordinate $(t,x) \in [0,1] \times I_-$ such that the augmented flow map lies in inappropriate half-space, i.e., $\Phi(t,x) \in H_-^c$. This means the flow map $\phi(t,x) \in \R$ is above the line $\ell(x_0, s_0)$ at time $t$. Since $x < x_0$ and the curve $t \mapsto \phi(t,x)$ is continuous, it must cross the line $\ell(x_0, s_0)$ at some time $s < t$ yielding $\phi(s, x) = \phi(s, x_0)$ thus violating injectivity.

  We now establish the validity of~\eqref{eq:edge-point-control-mixture-to-mixture-d-1}. Consider first the case $i = 0$ and suppose, for the purpose of contradiction, that $\phi(1, x_0) < s_0$. By the injectivity and continuity of the flow (c.f., the above paragraph) it follows that $\phi(1, \cdot)$ is monotone: it must be that $\phi(1, x) < s_0$ for all $x < x_0$ and similarly $\phi(1, x) > s_0$ for all $x > x_0$.
  Therefore
  \begin{align}
    \mathbb{P}\Big( \phi(1, X_0) \leq \phi(1, x_0) \Big) &= \mathbb{P}\Big( X_0 \leq x_0 \Big) \nonumber\\
    &= P_0\Big((-\infty, x_0]\Big) \nonumber\\
    &= P_1(S_0) \, ,\label{eqn:greater_p1_s0}
  \end{align}
  where in the last step we have used the choice of $x_0$. By Corollary~\ref{corollary:generalized-interpolant-push-messure-forward}, however, we have
  \begin{equation*}
    X_t \overset{d}{=} \phi(t, X_0) \, , \quad \textup{ for all } t \in [0,1] \, ,
  \end{equation*}
  for $X_\bullet$ the process inducing the flow $\phi$.
  Since $X_1 \sim P_1$ we obtain
  \begin{align*}
    \mathbb{P}\Big( \phi(1, X_0) \leq \phi(1, x_0) \Big) &= \mathbb{P}\Big( X_1 \leq \phi(1, x_0) \Big) \\
    &= \mathbb{P}\Big(X_1 \leq s_0\Big) - \mathbb{P}\Big(\phi(1, x_0) < X_1 \leq s_0 \Big) \\
    &< P_1(S_0) \, ,
  \end{align*}
  and in the last step we have used that properness of the support of $P_1$,

  c.f. Definition~\ref{def:epsilon-disconnected-support}, implying that $P_1$ is positive on the interval $(\phi(1, x_0), s_0)$.
  This yields a contradiction with~\eqref{eqn:greater_p1_s0}.

  On the other hand, consider instead assuming $\phi(1, x_0) > s_0$ for contradiction. Then, by the continuity of the map $x \mapsto \phi(1, x)$, there is a real number $h > 0$ such that, if $y \in (x_0 - h, x_0 + h)$, then $\phi(1, y) > s_0$.
  By the monotonicity of $\phi(1, \cdot)$ discussed above, we can further conclude that having $y \in (x_0 - h, x_0)$ implies property $\phi(1, x) > \phi(1, y) > s_0$.
  Thus, we have
  \begin{align*}
    \mathbb{P}\Big( \phi(1, X_0) \leq \phi(1, x_0) \Big) &= \mathbb{P}\Big( \phi(1, X_0) \leq s_0 \Big) - \mathbb{P}\Big( s_0 \leq \phi(1, X_0) \leq \phi(1, x_0) \Big) \\
    &\leq \mathbb{P}\Big( \phi(1, X_0) \leq s_0 \Big) - \mathbb{P}\Big( x_0 - h \leq X_0 \leq x_0 \Big) \\
    &< \mathbb{P}\Big( \phi(1, X_0) \leq s_0 \Big) \, ,
  \end{align*}
  using positivity of $P_0$ in the last step.
  Recall that Corollary~\ref{corollary:generalized-interpolant-push-messure-forward} gives $\phi(1, \cdot)_\sharp P_0 = P_1$, and so the last inequality can be rewritten as
  \begin{align*}
    \mathbb{P}\Big( \phi(1, X_0) \leq \phi(1, x_0) \Big) < \mathbb{P}\Big( X_1 \leq s_0 \Big) = P_1(S_0) \, .
  \end{align*}
  The monotonicity of $\phi(1, \cdot)$, however, implies that
  \begin{equation*}
    \mathbb{P}\Big( \phi(1, X_0) \leq \phi(1, x_0) \Big) = \mathbb{P}\Big( X_0 \leq x_0 \Big) = P_1(S_0) \, ,
  \end{equation*}
  by the choice of $x_0$, yielding a contradiction and completing the proof of~\eqref{eq:edge-point-control-mixture-to-mixture-d-1} for $i = 0$.
  The case $i = 1$ is handled similarly.
  This completes the proof of~\eqref{eq:I-to-H-seperation-mixture-of-Gaussians}.

  \subsubsection*{Step 3: Low-go zone}
  Define the \emph{low-go zone} and its time-slices, for each $t \in [0,1]$, by
  \begin{align*} 
    G = \big( H_+ \cup H_- \big)^c \quad\textup{and}\quad G_t = \big\{ x \in \R \, : \, (t,x) \in G \big\} \, .
  \end{align*}
  Recall that Corollary~\ref{corollary:generalized-interpolant-push-messure-forward} gives us the distributional equality
  \begin{equation}
    \label{eq:interpolant-property-mixture-of-Gaussians}
    X_t \overset{d}{=} \phi(t, X_0) \, , \quad \textup{ for all } t \in [0,1] \, .
  \end{equation}
  Now, we have
  \begin{align*}
    \mathbb{P} \Big( X_t \in G_t \Big) &= \mathbb{P} \Big( \phi(t, X_0) \in G_t \Big) \\
    &= \mathbb{P} \Big( \Phi(t, X_0) \in G \Big) \\
    &= \mathbb{P} \Big( \Phi(t, X_0) \in (H_+ \cup H_-)^c \Big) \\
    &\leq \mathbb{P} \Big( X_0 \in \big(I_- \cup I_+ \big)^c \Big) \\
  \end{align*}
  using equality~\eqref{eq:interpolant-property-mixture-of-Gaussians} in the first line, the definition of $\Phi$ in the second line, the definition of $G$ in the third line and~\eqref{eq:I-to-H-seperation-mixture-of-Gaussians} in the fourth line.
  Recalling the definition of $I_-$ and $I_+$ we have
  \begin{equation*}
    \big( I_- \cup I_+ \big)^c = I_-^c \cap I_+^c = \big\{ x \in \R \, : \, x_0 < x < x_1 \big\} \, ,
  \end{equation*}
  and recalling the definition of $x_0$ and $x_1$ we get
  \begin{equation*}
    \big( I_- \cup I_+ \big)^c = \big( S_0 \cup S_1 \big)^c \, .
  \end{equation*}
  Since $S_0$ and $S_1$ are as in Definition~\ref{def:disconnected-support} we have that $P_0(S_0 \cup S_1) \geq 1 - \epsilon$ and so
  \begin{equation}
    \label{eq:low-go-zone-prob-for-each-t}
    \mathbb{P} \Big( X_t \in G_t \Big) \leq \epsilon \, , \quad \textup{ for all } t \in [0,1] \, .
  \end{equation}

  \subsubsection*{Step 4: Small probability of crossing}

  Since $H_-$ and $H_+$ are disconnected and closed, any continuous curve that starts in $H_-$ and ends in $H_+$ must cross $G = (H_- \cup H_+)^c$.
  For $t \leq t' \in [0,1]$ consider the event of crossing the low-go zone $G$ between times $t$ and $t'$, i.e.,
  \begin{equation*}
    C_{t,t'} = \Big\{ \exists \, u_+, u_- \in [t,t'] \, : \, (u_+, X_{u_+}) \in H_+ \, , \, (u_-, X_{u_-}) \in H_- \Big\} \, ,
  \end{equation*}
  In this step, we use~\eqref{eq:low-go-zone-prob-for-each-t} to construct a function $g: \R_+ \to \R_+$ such that $g(\epsilon) \searrow 0$ as $\epsilon \searrow 0$ and
  \begin{equation}
    \label{eq:low-prob-of-crossing-low-go-zone}
    \mathbb{P} ( C_{0,1} ) = g(\epsilon) \, .
  \end{equation}
  Let $\delta > 0$ be a parameter to be chosen later.
  By a union bound, we have that
  \begin{equation}
    \label{eq:union-bound-C-event}
    \mathbb{P}(C_{0,1}) \leq \mathbb{P}(C_{0,\delta}) + \mathbb{P}(C_{\delta,1}) + \mathbb{P}(X_\delta \in G_\delta) \, ,
  \end{equation}
  where we have used that any crossing that starts in $(0, \delta)$ and ends in $(\delta, 1)$ must be in the low-go zone $G_\delta$ at time $\delta$. Now, the inequality in~\eqref{eq:low-go-zone-prob-for-each-t} allows us to immediately control the last term in~\eqref{eq:union-bound-C-event}, yielding
  \begin{equation}
    \label{eq:union-bound-C-term-3}
    \mathbb{P}(X_\delta \in G_\delta) \leq \epsilon \, .
  \end{equation}

  For the second term, we argue as follows: recall the definition of $G$ and that of $H_-$ and $H_+$ as the hypograph and epigraph of the line-segments $\ell(x_0, s_0)$ and $\ell(x_1, s_1)$ respectively. We now see that $G$ is a space-time trapezoid, defined by the non-intersecting line-segments $\ell(x_0, s_0)$ and $\ell(x_1, s_1)$ as well as the parallel lines $\big\{t(0, x_0 - x_1) \, : \, t \in [0,1] \big\}$ and $\big\{t(1, s_0 - s_1) \, : \, t \in [0,1] \big\}$.
  Therefore, we can conclude that for any $t \in [0,1]$ there are real numbers $a_t < b_t$ such that
  \begin{equation*}
    G_t = \{ x \in \R \, : \, a_t < x < b_t \} \, .
  \end{equation*}

  For a fixed $\epsilon > 0$ we can distinguish two cases:
  \subsubsection*{Case 1: $|x_0 - x_1| < \frac{1}{2}|s_0 - s_1|$}
    Here, the line-segments $\ell(x_0, s_0)$ and $\ell(x_1, s_1)$ diverge as $t$ increases, yielding the estimate
    \begin{equation*}
      | G_t | \geq | G_\delta | \geq \mu \, \delta \, , \quad \textup{for all} \;\; t \in [\delta, 1] \, ,
    \end{equation*}
    with
    \begin{equation*}
      \mu = \big| s_1 - s_0 - (x_1 - x_0) \big| > \frac{| s_1 - s_0 | }{2}> 0 \, ,
    \end{equation*}
    using the reverse triangle inequality.
    Finally, using Lemma~\ref{lemma:modulus-of-continuity-control}, we obtain
    \begin{align*}
      \mathbb{P}\Big( C_{\delta, 1} \Big) &\leq \mathbb{P}\Big( N_{X_\bullet}(a_\delta, b_\delta) \geq 1 \Big) \\
      &\leq A \left( \frac{\epsilon^\alpha}{|G_\delta|^\beta} \right)^{\frac{1}{\alpha + 1}} \, ,
    \end{align*}
    and, using the above lower bounds on $|G_\delta|$ and the constant $\mu$, we obtain
    \begin{equation*}
      \mathbb{P}\Big( C_{\delta, 1} \Big) \leq \frac{2^{\frac{\beta}{\alpha + 1}} A}{|s_0 - s_1|^\frac{\beta}{\alpha + 1}} \left( \frac{\epsilon^\alpha}{\delta^\beta} \right)^{\frac{1}{\alpha + 1}} \, .
    \end{equation*}

  \subsubsection*{Case 2: $|x_0 - x_1| \geq \frac{1}{2}|s_0 - s_1|$}
    Here, the line-segments $\ell(x_0, s_0)$ and $\ell(x_1, s_1)$ converge as $t$ increases, so we have the estimate
    \begin{equation*}
      |G_t| \geq |G_1| = \frac{|s_0 - s_1|}{2} \, , \quad \textup{for all} \;\; t \in [\delta, 1] \, ,
    \end{equation*}
    Therefore, using Lemma~\ref{lemma:modulus-of-continuity-control} we obtain
    \begin{equation*}
      \mathbb{P}\Big( C_{\delta, 1} \Big) \leq A \left( \frac{\epsilon^\alpha}{|G_1|^\beta} \right)^{\frac{1}{\alpha + 1}} = \frac{2^{\frac{\beta}{\alpha + 1}} A}{|s_0 - s_1|^\frac{\beta}{\alpha + 1}} \, \epsilon^{\frac{\alpha}{\alpha + 1}} \, .
    \end{equation*}
  Therefore, we can combine the above two cases to obtain the following upper bound on the second term in the union bound~\eqref{eq:union-bound-C-event}:
  \begin{equation}
    \label{eq:union-bound-C-term-2}
    \mathbb{P}\Big( C_{\delta, 1} \Big) \leq \frac{2^{\frac{\beta}{\alpha + 1}} A}{|s_0 - s_1|^\frac{\beta}{\alpha + 1}} \left(\left( \frac{\epsilon^\alpha}{\delta^\beta} \right)^{\frac{1}{\alpha + 1}} + \epsilon^{\frac{\alpha}{\alpha + 1}} \right) \, .
  \end{equation}

  We now turn our attention to the first term $\mathbb{P}(C_{0, \delta})$. We shall now control this term as $\delta \to 0$.\footnote{
    Note that upper continuity of the measure $\mathbb{P}$ implies that
    \begin{equation}
      \label{eq:union-bound-C-term-1}
      \lim_{\delta \to 0} \mathbb{P}(C_{0, \delta}) = \mathbb{P}\left( \bigcap_{\delta > 0} C_{0, \delta} \right) = \mathbb{P}( C_{0,0} ) = 0 \, ,
    \end{equation}
    and the last step follows by the definition of the event $C_{0, 0}$. However, to proceed we need \emph{quantitative control} on the rate at which $\mathbb{P}(C_{0, \delta})$ goes to zero as $\delta \to 0$.
  }
  Introduce a parameter $\xi > 0$ to be chosen later and decompose the event $C_{0, \delta}$ into the following cover:
  \begin{align*}
    C_{0, \delta}^{\xi_-} \coloneq C_{0, \delta} \cap \Big\{ \omega : \, \Big|X_0(\omega) - \frac{x_0 + x_1}{2}\Big| \leq \xi \Big\} \, , \quad C_{0, \delta}^{\xi_+} \coloneq C_{0, \delta} \cap \Big\{ \omega : \, \Big|X_0(\omega) - \frac{x_0 + x_1}{2} \Big| > \xi \Big\} \, .
  \end{align*}
  Since the events $C_{0, \delta}^{\xi_-}$ and $C_{0, \delta}^{\xi_+}$ form a partition of $C_{0, \delta}$ we have
  \begin{align*}
    \mathbb{P}(C_{0, \delta}) &= \mathbb{P}\Big( C_{0, \delta}^{\xi_-} \Big) + \mathbb{P}\Big( C_{0, \delta}^{\xi_+} \Big)
  \end{align*}
  and using the $(C, \gamma)$-Frostman assumption on $P_0$ we have
  \begin{equation*}
    \mathbb{P}\Big( C_{0, \delta}^{\xi_-} \Big) \leq \mathbb{P}\Big( \Big|X_0 - \frac{x_0 + x_1}{2}\Big| \leq \xi \Big) \leq C \, \xi^\gamma \, .
  \end{equation*}
  On the other hand, we can use that $X_\bullet \in \mathcal{C}(A, \alpha, \beta)$ and Corollary~\ref{cor:time-dependent-up-crossing-lemma} to obtain
  \begin{align*}
    \mathbb{P}\Big( C_{0, \delta}^{\xi_+} \Big) &\leq \mathbb{P}\left( N_{X_\bullet}\left( \frac{x_0 + x_1}{2} - \xi, \frac{x_0 + x_1}{2}\right) \geq 1 \right) + \mathbb{P}\left( N_{X_\bullet}\left( \frac{x_0 + x_1}{2}, \frac{x_0 + x_1}{2} + \xi \right) \geq 1 \right) \\
    &\leq 2 \, A \left( \frac{\delta^\alpha}{\xi^\beta} \right)^{\frac{1}{\alpha + 1}} \, .
  \end{align*}
  Therefore, we obtain the upper bound
  \begin{equation*}
    \mathbb{P}(C_{0, \delta}) \leq C \, \xi^\gamma + 2 \, A \frac{\delta^\frac{\alpha}{\alpha + 1}}{\xi^\frac{\beta}{\alpha+1}} \, ,
  \end{equation*}
  and minimizing the right-hand side over $\xi > 0$ we obtain
  \begin{equation*}
    \mathbb{P}(C_{0, \delta}) \leq B(\alpha, \beta, \gamma, A, C) \, \delta^{\frac{\alpha \, \gamma}{\gamma (\alpha + 1) + \beta} } \, ,
  \end{equation*}
  for a constant $B(\alpha, \beta, \gamma, A, C) > 0$ that depends only on the parameters $\alpha, \beta, \gamma, A$ and $C$.\footnote{
    In fact, we have the explicit expression
    \begin{equation*}
      B(\alpha, \beta, \gamma, A, C) = \left[ \left(\frac{\frac{\beta}{\alpha} + 1}{\gamma}\right)^{\frac{\gamma}{\gamma + \frac{\beta}{\alpha + 1}}} + \left( \frac{\gamma}{\frac{\beta}{\alpha + 1}} \right)^{\frac{\frac{\beta}{\alpha + 1}}{\gamma + \frac{\beta}{\alpha + 1}}} \right] C^{\frac{\frac{\beta}{\alpha+1}}{\gamma + \frac{\beta}{\alpha + 1}}} \left( 2\, A \right)^{\frac{\gamma}{\gamma + \frac{\beta}{\alpha + 1}}} \, .
    \end{equation*}
  }

  We can now select
  \begin{equation*}
    \delta(\epsilon) = \epsilon^{\frac{\alpha}{2 \, \beta}} \, ,
  \end{equation*}
  and combine equation~\eqref{eq:union-bound-C-term-3}, equation~\eqref{eq:union-bound-C-term-2} and equation~\eqref{eq:union-bound-C-term-1} in the union bound~\eqref{eq:union-bound-C-event} to obtain
  \begin{equation}
    \label{eq:step-4-gaussian-mixture-quant-control}
    \mathbb{P}(C_{0,1}) \leq D(\alpha, \beta, \gamma, A, C, S_0, S_1) \left( \epsilon^{\frac{\alpha^2 \, \gamma}{2 \, \beta \, \gamma (\alpha + 1) + 2 \, \beta^2} } + \epsilon^{\frac{\alpha}{2(\alpha + 1 )}} + \epsilon^{\frac{\alpha}{\alpha + 1}} + \epsilon \right) \, .
  \end{equation}
  and $D(\alpha, \beta, \gamma, A, C, S_0, S_1) > 0$ is a constant that depends only on the parameters $\alpha, \beta, \gamma, A, C$ and the sets $S_0$ and $S_1$.
  We have, therefore, completed the proof of the asserted equality~\eqref{eq:low-prob-of-crossing-low-go-zone} with the function $g$ defined as the right-hand side of the above display.

  \subsubsection*{Step 5: Contradiction}
  To conclude the proof, we note that by independence we have
  \begin{equation*}
    \mathbb{P}\Big( X_0 \leq x_0 \text{ and } X_1 \in S_1 \Big) = \mathbb{P}(X_0 \leq x_0) \, \mathbb{P}(X_1 \in S_1) = P_1(S_0) \, P_1(S_1) \, ,
  \end{equation*}
  On the other hand, we have the event inclusions
  \begin{equation*}
    \big\{ X_0 \leq x_0 \big\} \subseteq \big\{ (0, X_0) \in H_- \big\} \quad \text{and} \quad \big\{ X_1 \in S_1 \big\} \subseteq \big\{ (1, X_1) \in H_+ \big\} \, ,
  \end{equation*}
 therefore, we can write
  \begin{align*}
    \mathbb{P}\Big( X_0 \leq x_0 \text{ and } X_1 \in S_1 \Big) &\leq \mathbb{P}\Big( (0, X_0) \in H_- \text{ and } (1, X_1) \in H_+ \Big) \\
    &\leq \mathbb{P}\Big( C_{0,1} \Big) \, .
  \end{align*}
  Set $w_i \coloneq P_1(S_i)$ for $i = 0,1$ as in the statement of this theorem.
  Using equation~\eqref{eq:step-4-gaussian-mixture-quant-control} and the above displays we obtain the bound
  \begin{equation}
    \label{eq:necessary-inequality-gaussian-mixture-proof}
    w_0 \, w_1 \leq D(\alpha, \beta, \gamma, A, C, S_0, S_1) \left( \epsilon^{\frac{\alpha^2 \, \gamma}{2 \, \beta \, \gamma (\alpha + 1) + 2 \, \beta^2} } + \epsilon^{\frac{\alpha}{2(\alpha + 1 )}} + \epsilon^{\frac{\alpha}{\alpha + 1}} + \epsilon \right) \, ,
  \end{equation}
  Since all exponents of $\epsilon$ in the right-hand side are positive the right-hand side a decreasing function of $\epsilon$ that goes to zero as $\epsilon \searrow 0$.

  Therefore, there exists an $\epsilon_0 = \epsilon_0(\alpha, \beta, \gamma, A, C, S_0, S_1) > 0$ such that
  \begin{equation*}
    w_0 \, w_1 > D(\alpha, \beta, \gamma, A, C, S_0, S_1) \left( \epsilon_0^{\frac{\alpha^2 \, \gamma}{2 \, \beta \, \gamma (\alpha + 1) + 2 \, \beta^2} } + \epsilon_0^{\frac{\alpha}{2(\alpha + 1 )}} + \epsilon_0^{\frac{\alpha}{\alpha + 1}} + \epsilon_0 \right) \, .
  \end{equation*}
  Clearly, for any $\epsilon \leq \epsilon_0$ the necessary inequality~\eqref{eq:necessary-inequality-gaussian-mixture-proof} is violated, yielding a contradiction and completing the proof of the theorem.
\end{proof}

\CorNormalSourceNonExist*
\begin{proof}
  Note that $P_0$ is a $(C, \gamma)$-Frostman measure for explicit constants $C, \gamma > 0$ since it has a bounded density, see Example~\ref{ex:Lp-are-Frostman}.
  Now we can select $\mu_i \in \R$ such that the sets $S_0 = \{ x \in \R \, : \, |x - \mu_0| < 1 \}$ and $S_1 = \{ x \in \R \, : \, |x - \mu_1| < 1\}$ are disjoint. Now for any $\epsilon > 0$, there exist $\sigma_0, \sigma_1 > 0$ such that the measure $P_1$ defined as above has properly $\epsilon$-disconnected support on the sets $S_0$ and $S_1$, as per Definition~\ref{def:epsilon-disconnected-support}.
  By Theorem~\ref{thm:non-existence-for-Gaussian-mixture}, there exists $\epsilon_0 = \epsilon_0\big(C, \gamma, S_0, S_1, w_0, w_1, A, \alpha, \beta\big) > 0$ such that if $\epsilon < \epsilon_0$ then we have
  \begin{equation*}
    \mathcal{S}_{\textup{SL}}(P_0, P_1) \cap \mathcal{C}(A, \alpha, \beta) = \emptyset \, .
  \end{equation*}
  In particular, we can select $\sigma_0, \sigma_1 > 0$ such that $\epsilon = \frac{\epsilon_0}{2} < \epsilon_0$ thereby completing the proof.
\end{proof}
 
\section{Auxiliary Technical Results}\label{appendix-C}
\begin{lemma}
  \label{lemma:generalized-interpolant-continuity-equation}
  Consider any process $X_\bullet \in \mathcal{S}(P_0, P_1)$ and let $v$ be the associated velocity field and $P_t$ be the marginal law of $X_t$ for all $t \in [0,1]$.
  Then, the pair $(P, v)$ consisting of the time-marginal laws $P_t$ of $X_t$ and the velocity field $v_t$ satisfy the measure valued continuity equation
  \begin{equation*}
    \partial_t P_t + \nabla \cdot (P_t \, v_t) = 0 \, ,
  \end{equation*}
  in the sense of~\cite[Section 8.1]{ambrosio2005gradient}.
\end{lemma}

\begin{proof}
  The proof extends ideas in \cite{albergo2023stochastic}.
  First, note that since $v_t$ is defined by
  \begin{equation*}
    v_t(X_t) = \mathbb{E}[\dot X_t \mid X_t],
  \end{equation*}
  for every bounded measurable vector field $g : \mathbb{R}^d \to \mathbb{R}^d$ we have
  \begin{equation*}
    \mathbb{E}\big[g(X_t) \cdot \dot X_t \big]
    =
    \mathbb{E}\big[g(X_t) \cdot v_t(X_t) \big].
  \end{equation*}
  In particular,
  \begin{equation}
    \label{eq:v_t-P_t-integrable}
    \int_{\mathbb{R}^d} \|v_t(x)\| \, \dd P_t(x)
    =
    \mathbb{E}[\|v_t(X_t)\|]
    \leq
    \mathbb{E}[\|\dot X_t\|],
  \end{equation}
  so $P_t v_t$ is a well-defined vector-valued since
  $\mathbb{E}[\|\dot X_t\|] < \infty$ by Definition~\ref{def:S-P_0-P_1}.

  We now prove the continuity equation in the distributional sense. Let
  $\psi \in C_c^\infty((0,1)\times \mathbb{R}^d)$. Since the sample paths of
  $X_\bullet$ are absolutely continuous, the map
  \begin{equation*}
    t \mapsto \psi(t,X_t)
  \end{equation*}
  is absolutely continuous almost surely, and the chain rule gives
  \begin{equation*}
    \frac{\dd}{\dd t}\psi(t,X_t)
    =
    \partial_t \psi(t,X_t)
    +
    \nabla_x \psi(t,X_t) \cdot \dot X_t
  \end{equation*}
  for almost every $t \in (0,1)$, see~\cite[Section 4]{evans2025measure} for details on weak differentiability.
  Integrating in time and taking expectations, we obtain
  \begin{align*}
    0
    &=
    \mathbb{E}\left[
      \int_0^1
      \frac{\dd}{\dd t}\psi(t,X_t)
      \, \dd t
    \right] \\
    &=
    \int_0^1
    \mathbb{E}\left[
      \partial_t \psi(t,X_t)
      +
      \nabla_x \psi(t,X_t) \cdot \dot X_t
    \right]
    \, \dd t.
  \end{align*}
  The boundary term vanishes because $\psi$ is compactly supported in
  $(0,1)\times \mathbb{R}^d$.

  Using the defining property of conditional expectation, with
  $g_t(x)=\nabla_x \psi(t,x)$, we have
  \begin{equation*}
    \mathbb{E}\left[
      \nabla_x \psi(t,X_t) \cdot \dot X_t
    \right]
    =
    \mathbb{E}\left[
      \nabla_x \psi(t,X_t) \cdot v_t(X_t)
    \right].
  \end{equation*}
  Therefore
  \begin{align*}
    0
    &=
    \int_0^1
    \mathbb{E}\left[
      \partial_t \psi(t,X_t)
      +
      \nabla_x \psi(t,X_t) \cdot v_t(X_t)
    \right]
    \, \dd t \\
    &=
    \int_0^1
    \int_{\mathbb{R}^d}
    \left[
      \partial_t \psi(t,x)
      +
      \nabla_x \psi(t,x) \cdot v_t(x)
    \right]
    \, \dd P_t(x)
    \, \dd t.
  \end{align*}
  This is precisely the distributional formulation of
  \begin{equation*}
    \partial_t P_t + \nabla \cdot (P_t v_t) = 0
  \end{equation*}
  on $(0,1)\times \mathbb{R}^d$.
Thus $(P,v)$ solves the measure-valued continuity equation in the distributional
  sense.
\end{proof}

\begin{corollary}
  \label{corollary:generalized-interpolant-push-messure-forward}
  Consider any $X_\bullet \in \mathcal{S}(P_0, P_1)$ and let $v$ and $\phi$ be the associated velocity field and flow map, respectively, see Definition~\ref{def:straight-line-process-and-straight-line-interpolant} for details.
  Then, we have the distributional equality
  \begin{equation*}
    \phi(t, X_0) \overset{d}{=} X_t \, \quad \textup{ for all } t \in [0,1] \, .
  \end{equation*}
  In particular, we have $X_1 \overset{d}{=} \phi(1, X_0)$ and so
  \begin{equation*}
    P_1 = \phi(1, \cdot)_\sharp \, P_0 \, .
  \end{equation*}
\end{corollary}
\begin{proof}
  This is a direct application of the so called \emph{Representation formula for the continuity equation}~\cite[Proposition 8.1.8]{ambrosio2005gradient}.
  To justify using this proposition we note that: (i) by Lemma~\ref{lemma:generalized-interpolant-continuity-equation} the pair $(P, v)$ satisfies the measure valued continuity equation; 
  (ii) by Definition~\ref{def:S-P_0-P_1} the equation~\cite[8.1.7]{ambrosio2005gradient} is satisfied; (iii) by equation~\eqref{eq:v_t-P_t-integrable} the integrability condition~\cite[8.1.2]{ambrosio2005gradient} is satisfied; and (iv) the curve $t \mapsto P_t$ is narrowly continuous, as we now show.

  Recall by Definition~\ref{def:S-P_0-P_1} that $X_\bullet$ has continuous sample paths.\footnote{
    Recall that processes $X_\bullet \in \textup{S}(P_0, P_1)$ were assumed to have paths in $W^{1,1}([0,1]; \Rd)$ almost surely. It is classical that functions in $W^{1,1}([0,1]; \R)$ have continuous representatives, see~\cite[Theorem 8.2]{brezis2011functional}.
  } 
  Now, for any test function $\varphi \in C_b(\mathbb{R}^d)$, we have $\varphi(X_s) \to \varphi(X_t)$
  almost surely as $s \to t$ due to the sample path continuity of $X_\bullet$, and therefore, by dominated convergence,
  \begin{equation*}
    \lim_{s \to t} \int_{\mathbb{R}^d} \varphi(x) \, \dd P_s(x)
    =
    \lim_{s \to t} \mathbb{E}[\varphi(X_s)]
    =
    \mathbb{E}[\varphi(X_t)]
    =
    \int_{\mathbb{R}^d} \varphi(x) \, \dd P_t(x).
  \end{equation*}
  Hence $t \mapsto P_t$ is narrowly continuous.  
  
\end{proof}

\section{Examples}\label{appendix-D}

\begin{example}
  \label{ex:concentration-of-interpolants}
  Concentration inequalities like the one in Definition~\ref{def:concentration-class} are typical for stochastic processes used in sampling. For example, consider the common family of processes
  \begin{equation*}
    X_t = a_t \, X_0 + b_t \, X_1 + c_t \, Z \quad \textup{ where } \quad (X_0, X_1, Z) \sim P_0 \otimes P_1 \otimes \mathcal{N}(0,1) \, ,
  \end{equation*}
  where $a, b, c \in C^{0, \eta}([0,1];\Rd)$ are H\"older continuous with exponent $0 < \eta < 1$ satisfying appropriate boundary conditions. Now assume that $P_0$ and $P_1$ have $\alpha$- and $\beta$-polynomial tails respectively and recall that $\mathcal{N}(0,1)$ has $\gamma$-polynomial tails for every $\gamma > 0$.

  Since, for any $f \in C^{0,\eta}([0,1]; \Rd)$, one has
  \begin{equation*}
    \kappa_f(\delta) \leq \delta^\eta \, \|f\|_{C^{0,\eta}} \, ,
  \end{equation*}
  we may set $M = \|\alpha\|_{C^{0,\eta}} \vee \|\beta\|_{C^{0,\eta}} \vee \|\gamma\|_{C^{0,\eta}} $ and obtain
  \begin{align*}
    \kappa_{X_\bullet}(\delta) &\leq \delta^\eta \, \left( \| \alpha \|_{C^{0,\eta}} \, |X_0| + \| \beta \|_{C^{0,\eta}} \, |X_1| + \| \gamma \|_{C^{0,\eta}} \, |Z| \right) \\
    &\leq \delta^\eta \, M \left( |X_0| + |X_1| + |Z| \right)
  \end{align*}
  which by a union bound leads to
  \begin{align*}
    \mathbb{P}\left( \kappa_{X_\bullet}(\delta) \geq \theta \right) &\leq \mathbb{P}\left( |X_0| \geq \frac{\theta}{3 \, \delta^\eta M} \right) + \mathbb{P}\left( |X_1| \geq \frac{\theta}{3 \, \delta^\eta M} \right) + \mathbb{P}\left( |Z| \geq \frac{\theta}{3 \, \delta^\eta M} \right) \\
    & \lesssim \left(\frac{\theta}{\delta^\eta}\right)^{-\alpha} + \left(\frac{\theta}{\delta^\eta}\right)^{-\beta} + \left(\frac{\theta}{\delta^\eta}\right)^{-\gamma} \\
      & \lesssim \left(\frac{\delta^\eta}{\theta}\right)^{\alpha \wedge \beta \wedge \gamma} \\
  & = \frac{\delta^{\eta \, (\alpha \wedge \beta \wedge \gamma)}}{\theta^{\alpha \wedge \beta \wedge \gamma}} \, .
  \end{align*}
  Thus, with $\tilde \alpha \coloneq \eta \, (\alpha \wedge \beta \wedge \gamma)$ and $\tilde \beta \coloneq \alpha \wedge \beta \wedge \gamma$ and if $A$ is the numerical constant suppressed in the above inequalities, we have shown that
  \begin{equation*}
    X_\bullet \in \mathcal{C}(A, \tilde \alpha, \tilde \beta) \, .
  \end{equation*}
\end{example}

\begin{example}
  \label{ex:Lp-are-Frostman}
  Frostman measures are well-studied in the literature, see for example~\cite[Theorem 8.8]{mattila2012geometry}. In fact, the class of Frostman measures contains absolutely continuous measures with $L^p(\Rd)$ densities for any $p \in (1, \infty]$. Indeed, for $p= \infty$ and a measure $P$ with density $\pi \in L^\infty(\Rd)$ we have
  \begin{align*}
    P(B_r(x)) = \int_{B_r(x)} \pi(y) \, dy \leq \|\pi\|_\infty \, | B_r(x) | = \|\pi\|_\infty \, | B_1(0) | \, r^d \, ,
  \end{align*}
  where $|B_r(x)|$ is the Lebesgue measure of the ball $B_r(x)$.
  Thus, measure $P$ is $(C, d)$-Frostman with $C = \|\pi\|_\infty \, | B_1(0) |$. For $p \in (1, \infty)$ we can apply H\"older's inequality to obtain
  \begin{align*}
    P(B_r(x)) = \int_{B_r(x)} \pi(y) \, dy \leq \|\pi\|_p \, | B_r(x) |^{1 - \frac{1}{p}} = \|\pi\|_p \, | B_1(0) |^{1 - \frac{1}{p}} \, r^{\frac{d(p - 1)}{p}} \, ,
  \end{align*}
  thus $P$ is $(C, \frac{d(p - 1)}{p})$-Frostman with $C = \|\pi\|_p \, | B_1(0) |^{1 - \frac{1}{p}}$.
\end{example}

\begin{remark}\label{rk:trapezoidal-space-time-geometry}
  We point out that an essential ingredient to the proof above was the scaling
  \begin{equation*}
    |G_\delta| \asymp \delta \, ,
  \end{equation*}
  in the case $|x_0 - x_1| < \frac{1}{2}|s_0 - s_1|$--this is a consequence of the geometry of the low-go zone $G$ describing a trapezoidal space-time region.
  Under the same assumption $|x_0 - x_1| < \frac{1}{2}|s_0 - s_1|$, if one tried to \emph{erroneously} do the above proof by restricting their attention to the smaller rectangular space-time region
  \begin{equation*}
    I(\epsilon) = [0,1] \times [x_0(\epsilon), x_1(\epsilon) ] \subsetneq G \, ,
  \end{equation*}
  recalling the explicit dependence of $x_i = x_i(\epsilon)$ for $i = 0,1$, see the beginning of the proof above.
  Suppose for concreteness that $P_0$ is a standard normal distribution and that $S_0 = -S_1$ are symmetric around zero.
  Then, $x_0(\epsilon) = -x_1(\epsilon)$ and $x_0(\epsilon)$ is the unique solution to the equation
  \begin{equation}
    \label{eq:gaussian-x0}
    \frac{1}{\sqrt{2 \pi}} \int_{-\infty}^{x_0(\epsilon)} e^{- \frac{x^2}{2}} \, dx = \frac{1 - \epsilon}{2} \, .
  \end{equation}
  By Lemma~\ref{lemma:modulus-of-continuity-control} we have
  \begin{equation}
    \label{eq:gaussian-erroneous-control}
    \mathbb{P}\Big(N_{X_\bullet}(-x_0(\epsilon), x_0(\epsilon))\Big) \leq A \left( \frac{\epsilon^\alpha}{2 \, x_0^\beta(\epsilon)} \right)^{\frac{1}{\alpha + 1}} \, ,
  \end{equation}
  and using equation~\eqref{eq:gaussian-x0} we can straightforwardly compute that
  \begin{equation*}
    \lim_{\epsilon \downarrow 0} \frac{x_0(\epsilon)}{\epsilon} = c \neq 0 \, .
  \end{equation*}
  Thus, for the right hand side of~\eqref{eq:gaussian-erroneous-control} to vanish as $\epsilon \searrow 0$, a \emph{necessary} ingredient of our proof strategy, it must be that $\alpha > \beta$.
  Is this reasonable? Looking at Example~\ref{ex:concentration-of-interpolants} we see that to enforce this for a typical sampling processes of the form $X_t = a_t \, X_0 + b_t \, X_1 + c_t \, Z$ we would need to take $a, b, c \in C^{0, \eta}([0,1];\Rd)$ with $\eta > 1$. A classical argument shows that on connected domains $\eta > 1$ implies that $a, b, c$ are constant functions, which is impossible since, in particular, we want $a_0 = 1$ and $a_1 = 0$. We conclude that the class $\mathcal{C}(A, \alpha, \beta)$ with $\alpha > \beta$ is somewhat trivial, since it contains none of the processes we typically use in sampling.
\end{remark}

\end{document}